\documentclass[nohyperref]{article}

\usepackage{microtype}
\usepackage{subfigure}

\usepackage[accepted]{icml2023}

\usepackage{amsmath,amsfonts,bm}

\def\eqref#1{equation~\ref{#1}}

\def\1{\bm{1}}

\def\rvs{{\mathbf{s}}}

\def\rvx{{\mathbf{x}}}
\def\rvy{{\mathbf{y}}}
\def\rvz{{\mathbf{z}}}

\DeclareMathAlphabet{\mathsfit}{\encodingdefault}{\sfdefault}{m}{sl}
\SetMathAlphabet{\mathsfit}{bold}{\encodingdefault}{\sfdefault}{bx}{n}

\def\gY{{\mathcal{Y}}}

\newcommand{\E}{\mathbb{E}}

\newcommand{\R}{\mathbb{R}}

\newcommand{\softmax}{\mathrm{softmax}}

\newcommand{\KL}{D_{\mathrm{KL}}}
\newcommand{\Var}{\mathrm{Var}}

\DeclareMathOperator*{\argmax}{arg\,max}
\DeclareMathOperator*{\argmin}{arg\,min}

\usepackage{hyperref}
\usepackage{url}
\usepackage{amssymb}
\usepackage{enumitem}
\usepackage{amsthm}
\usepackage{graphicx}
\usepackage{booktabs}
\usepackage{multirow}
\usepackage{rotating}
\usepackage{xcolor}

\usepackage[nameinlink]{cleveref}
\crefname{thmdef}{definition}{definitions}
\crefname{lemma}{lemma}{lemmas}
\crefname{prop}{proposition}{propositions}
\Crefname{algocf}{Algorithm}{Algorithms}

\newtheorem{prop}{Proposition}
\newtheorem{appprop}{Proposition}

\newtheorem{thm}{Theorem}
\newtheorem{appthm}{Theorem}

\usepackage[ruled,vlined]{algorithm2e}
\SetKwInput{KwInput}{Input}
\SetKwInput{KwOutput}{Output}
\newlength\myindent
\icmltitlerunning{Learning to Maximize Mutual Information for Dynamic Feature Selection}

\begin{document}

\twocolumn[
\icmltitle{Learning to Maximize Mutual Information for Dynamic Feature Selection}

\icmlsetsymbol{equal}{*}

\begin{icmlauthorlist}
\icmlauthor{Ian Covert}{cse}
\icmlauthor{Wei Qiu}{cse}
\icmlauthor{Mingyu Lu}{cse}
\icmlauthor{Nayoon Kim}{cse}
\icmlauthor{Nathan White}{med}
\icmlauthor{Su-In Lee}{cse}
\end{icmlauthorlist}

\icmlaffiliation{cse}{Paul G. Allen School of Computer Science \& Engineering, University of Washington}
\icmlaffiliation{med}{Department of Emergency Medicine, University of Washington}

\icmlcorrespondingauthor{Ian Covert}{icovert@cs.uw.edu}

\icmlkeywords{Machine Learning, ICML, Dynamic Feature Selection, Sequential Feature Selection, Active Feature Acquisition, Amortized Optimization, Differentiable Feature Selection, Concrete Distribution}

\vskip 0.3in
]



\printAffiliationsAndNotice{}  

\begin{abstract}
Feature selection helps reduce data acquisition costs in ML, but the standard approach is to train models with static feature subsets. Here, we consider the \textit{dynamic feature selection} (DFS) problem where a model sequentially queries features based on the presently available information. DFS is often addressed with reinforcement learning, but we explore a simpler approach of greedily selecting features based on their conditional mutual information. This method is theoretically appealing but requires oracle access to the data distribution, so we develop a learning approach based on amortized optimization. The proposed method is shown to recover the greedy policy when trained to optimality, and it outperforms numerous existing feature selection methods in our experiments, thus validating it as a simple but powerful approach for this problem.
\end{abstract} \label{sec:abstract}

\section{Introduction} \label{sec:introduction}
A machine learning model's inputs can be costly to obtain, and feature selection is often used to reduce data acquisition costs. In applications where information is gathered sequentially, a natural option is to select features adaptively based on the currently available information, rather than using a fixed feature set. This setup is known as \textit{dynamic feature selection} (DFS),\footnote{The problem has also been referred to as \textit{sequential feature selection}, \textit{active feature acquisition}, and \textit{information pursuit}.}
and the problem has been considered by several works in the last decade \citep{saar2009active, dulac2011datum, chen2015sequential, early2016test, he2016active, kachuee2018opportunistic}.

Compared to \textit{static} feature selection with a fixed feature set \citep{li2017feature, cai2018feature}, DFS can offer better performance given a fixed budget. This is easy to see, because selecting the same features for all instances (e.g., all patients visiting a hospital's emergency room) is suboptimal when the most informative features vary across individuals. Although it should in theory offer better performance, DFS also presents a more challenging learning problem, because it requires learning both (i)~a feature selection policy and (ii)~how to make predictions given variable feature sets.

Prior work has approached DFS in several ways, though often using reinforcement learning (RL) \citep{dulac2011datum, shim2018joint, kachuee2018opportunistic, janisch2019classification, li2021active}. RL is a natural approach for sequential decision-making problems, but current methods are difficult to train and do not reliably outperform static feature selection methods \citep{henderson2018deep, erion2021coai}. Our work therefore explores a simpler approach: greedily selecting features based on their conditional mutual information (CMI) with the response variable.

The greedy approach is known from prior work \citep{chen2015sequential, ma2019eddi}, but it is difficult to use in practice because calculating the CMI requires oracle access to the data distribution \citep{cover2012elements}. Our focus is therefore developing a practical approximation. Whereas previous work makes strong assumptions about the data \citep{geman1996active}
or approximates the data distribution with generative models \citep{ma2019eddi}, we develop a flexible approach that directly predicts the optimal selection at each step. Our method is based on a variational perspective on the greedy CMI policy, and it uses a technique known as \textit{amortized optimization} \citep{amos2022tutorial} to enable training using only a standard labeled dataset. Notably, the model is trained with an objective function that recovers the greedy policy when it is trained to optimality.

Our contributions in this work are the following:

\begin{enumerate}[leftmargin=0.5cm]
    \item We derive a variational, or optimization-based perspective on
    the greedy CMI policy, which shows it to be equivalent to minimizing the one-step-ahead prediction loss given an optimal classifier.

    \item We develop a learning approach based on amortized optimization, where a policy network is trained to directly predict the optimal selection at each step. Rather than requiring a dataset that indicates the correct selections, our training approach is based on
    a standard labeled dataset
    and an objective
    function
    whose global optimizer is the greedy CMI policy.

    \item We propose a continuous relaxation for the inherently discrete learning objective, which enables efficient and architecture-agnostic gradient-based optimization.
\end{enumerate}

Our experiments evaluate the proposed method on numerous datasets, and the results show that it outperforms many recent dynamic and static feature selection methods. Overall, our work shows that when learned properly, the greedy CMI policy is a simple and powerful approach for DFS.

\section{Problem formulation} \label{sec:formulation}
In this section, we describe the DFS problem and introduce notation used throughout the paper.

\subsection{Notation}

Let $\rvx$ denote a vector of input features and $\rvy$ a response variable for a supervised learning task. The input consists of $d$ distinct features, or $\rvx = (\rvx_1, \ldots, \rvx_d)$. We use the notation $s \subseteq [d] \equiv \{1, \ldots, d\}$ to denote a subset of indices and $\rvx_s = \{\rvx_i : i \in s\}$ a subset of features. Bold symbols $\rvx, \rvy$ represent random variables, the symbols $x, y$ are possible values, and $p(\rvx, \rvy)$ denotes the data distribution.

Our goal is to design a policy that controls which features are selected given the currently available information. The
selection policy can be viewed as a function $\pi(x_s) \in [d]$, meaning that it receives a subset of features as its input and outputs the next feature index to query. The policy is accompanied by a predictor $f(x_s)$ that can make predictions given the set of available features; for example, if $\rvy$ is discrete then predictions lie in the probability simplex, or $f(x_s) \in \Delta^{K - 1}$ for $K$ classes. The notation $f(x_s \cup x_i)$ represents the prediction given the combined features.
We initially consider policy and predictor functions that operate on feature subsets,
and \Cref{sec:method} proposes an implementation using a mask variable $m \in [0, 1]^d$ where the functions operate on $x \odot m$.

\subsection{Dynamic feature selection}

The goal of DFS is to select features with minimal budget that achieve maximum predictive accuracy. Having access to more features generally makes prediction easier, so the challenge is selecting a small number of informative features. There are multiple formulations for this problem, including versions with non-uniform feature costs and different budgets for each sample \citep{kachuee2018opportunistic}, but we focus on the setting with a fixed budget and uniform costs. Our goal is to handle predictions at inference time by beginning with no features, sequentially selecting features $x_s$ such that $|s| = k$ for a fixed budget $k < d$, and finally making accurate predictions for the response variable $y$.

Given a loss function that measures the discrepancy between predictions and labels $\ell(\hat{y}, y)$, a natural scoring criterion is the expected loss after selecting $k$ features. The scoring is applied to a policy-predictor pair $(\pi, f)$, and we define the score for a fixed budget $k$ as follows,
\begin{equation}
    v_k(\pi, f) = \E_{p(\rvx, \rvy)} \big[ \ell\big(f\big(\{\rvx_{i_t}\}_{t=1}^k \big), \rvy \big) \big], \label{eq:scoring}
\end{equation}
where feature indices are chosen sequentially for each
$(\rvx, \rvy)$ according to $i_n = \pi(\{ \rvx_{i_t} \}_{t = 1}^{n - 1})$.
The goal is to minimize $v_k(\pi, f)$, or equivalently, to maximize our final predictive accuracy.

One approach is to frame this as a Markov decision process (MDP) and solve it using standard RL techniques, so that $\pi$ and $f$ are trained to optimize a reward function based on \cref{eq:scoring}. Several recent works have designed such formulations \citep{shim2018joint, kachuee2018opportunistic, janisch2019classification, li2021active}. However, these approaches are difficult to train effectively,
so our work focuses on a greedy approach that is easier to learn and simpler to interpret.

\section{Greedy information maximization} \label{sec:greedy}
This section first defines the greedy CMI policy, and then describes an existing approximation strategy that relies on generative models.

\subsection{The greedy selection policy} \label{sec:cmi}

As an idealized approach to DFS, we are interested in the greedy algorithm that selects the most informative feature at each step. This feature can be defined in multiple ways, but we focus on the information-theoretic perspective that the most useful feature has maximum CMI with the response variable \citep{cover2012elements}. The CMI, denoted as $I(\rvx_i; \rvy \mid x_s)$, quantifies how much information an unknown feature $\rvx_i$ provides about the response $\rvy$ when accounting for the current features $x_s$, and it is defined as the KL divergence between the joint and factorized distributions:
\begin{equation*}
    I(\rvx_i; \rvy \mid x_s) = \KL \big( p(\rvx_i, \rvy \mid x_s) \mid\mid p(\rvx_i \mid x_s) p(\rvy \mid x_s) \big). \label{eq:cmi}
\end{equation*}
Based on this, we define the greedy CMI policy as $\pi^*(x_s) \equiv \argmax_i I(\rvx_i; \rvy \mid x_s)$, so that features are sequentially selected to maximize our information about the response variable. We can alternatively understand the policy as performing greedy uncertainty minimization, because this is equivalent to minimizing $\rvy$'s conditional entropy at each step, or $\pi^*(x_s) = \argmin_i H(\rvy \mid \rvx_i, x_s)$ \citep{cover2012elements}. For a complete characterization of this idealized approach,
we also consider that the policy is paired with the Bayes classifier as a predictor, or $f^*(x_s) = p(\rvy \mid x_s)$.

Maximizing the
information about $\rvy$ at each step is intuitive and should be effective
in many problems. Prior work has considered the same idea, but from two perspectives that differ from ours. First,
\citet{chen2015sequential} take a theoretical perspective and prove that the greedy algorithm
achieves performance within a multiplicative factor of the optimal policy;
the proof requires specific distributional assumptions, but we find that the greedy algorithm
performs well with many real-world datasets
(\Cref{sec:experiments}).
Second, from an implementation perspective, two works aim to provide practical approximations; however, these suffer from several limitations, so our work aims to develop a simple and flexible alternative (\Cref{sec:method}). In these works, \citet{ma2019eddi} and \citet{chattopadhyay2022interpretable} both require a conditional generative model of the data distribution, which we discuss next.

\subsection{Estimating conditional mutual information} \label{sec:iterative}

The greedy policy is trivial to implement if we can directly calculate CMI,
but this is rarely the case in practice. Instead, one option is to estimate it. We now describe a procedure to do so iteratively for each feature, assuming for now that we have oracle access to the response distributions $p(\rvy \mid \rvx_s)$ for all $s \subseteq [d]$ and the feature distributions $p(\rvx_i \mid \rvx_s)$ for all $s \subseteq [d]$ and $i \in [d]$.

At any point in the selection procedure, given the current features $x_s$, we can
estimate the CMI for a feature $\rvx_i$ where $i \notin s$ as follows. First, we can sample multiple values for $\rvx_i$ from its conditional distribution, or $x_i^j \sim p(\rvx_i \mid x_s)$ for $j \in [n]$. Next, we can generate Bayes optimal predictions for each sampled value, or $p(\rvy \mid x_s, x_i^j)$. Finally, we can calculate the mean prediction and the mean KL divergence relative to this prediction, which yields the following CMI estimator:
\begin{equation}
    I_i^n = \frac{1}{n} \sum_{j = 1}^n \KL\Big( p(\rvy \mid x_s, x_i^j) \mid\mid \frac{1}{n} \sum_{l = 1}^n p(\rvy \mid x_s, x_i^l) \Big). \label{eq:criterion}
\end{equation}
This score measures the variability among predictions and captures whether different $\rvx_i$ values significantly affect $\rvy$'s conditional distribution. The estimator can be used to select features, or we can set $\pi(x_s) = \argmax_i I_i^n$, due to the following limiting result (see \Cref{app:proofs}):
\begin{equation}
    \lim_{n \to \infty} I_i^n = I(\rvy; \rvx_i \mid x_s). \label{eq:limiting}
\end{equation}
This procedure thus provides a way to identify the correct greedy selections by estimating the CMI. Prior work has explored similar ideas for scoring features based on sampled predictions
\citep{saar2009active, chen2015value, early2016test, early2016dynamic}, but the implementation choices in these works prevent them from performing greedy information maximization. In \cref{eq:criterion}, is it important that our estimator uses the Bayes classifier, that we sample features from the conditional distribution $p(\rvx_i \mid x_s)$, and that we use the KL divergence as a measure of prediction variability. However, this estimator is impractical because we typically lack access to both $p(\rvy \mid \rvx_s)$ and $p(\rvx_i \mid \rvx_s)$.

In practice, we would instead require
learned substitutes for each distribution. For example, we can use a a classifier that approximates $f(x_s) \approx p(\rvy \mid x_s)$ and a generative model that approximates samples from $p(\rvx_i \mid \rvx_s)$. Similarly, \citet{ma2019eddi} propose jointly modeling $(\rvx, \rvy)$ with a conditional generative model, which is implemented via a modified VAE \citep{kingma2015variational}. This approach is limited for several reasons, including (i)~the difficulty of training an accurate conditional generative model, (ii)~the challenge of modeling mixed continuous/categorical features \citep{ma2020vaem, nazabal2020handling}, and (iii)~the slow CMI estimation process. In our approach, which we discuss next, we bypass all three of these challenges by directly predicting the best selection at each step.

\section{Proposed method} \label{sec:method}
We now introduce our approach, a practical approximation of the greedy policy trained using amortized optimization. Unlike prior work that estimates the CMI as an intermediate step, we develop a variational perspective on the greedy policy, which we then leverage to train a network that directly predicts the optimal selection given the current features.

\subsection{A variational perspective on CMI} \label{sec:variational}

For our purpose, it is helpful to recognize that the greedy policy can be viewed as the solution to an optimization problem. \Cref{sec:greedy} provides a conventional definition of CMI as a KL divergence, but this is difficult to integrate into an end-to-end learning approach. Instead, we now consider the one-step-ahead prediction achieved by a policy $\pi$ and predictor $f$, and we determine the behavior that minimizes their loss. Given the current features $x_s$ and a selection $i = \pi(x_s)$, the expected one-step-ahead loss is:
\begin{equation}
    \E_{\rvy,\rvx_i \mid x_s} \Big[ \ell\big(f(x_s \cup \rvx_i) , \rvy \big) \Big]. \label{eq:one-step}
\end{equation}
The variational perspective we develop here consists of two main results regarding this expected loss. The first result relates to the predictor, and we show that the loss-minimizing predictor can be defined independently of the policy $\pi$. We formalize this in the following proposition for classification tasks, and our results can also be generalized to regression tasks (see proofs in \Cref{app:proofs}).

\vskip 0.3cm
\begin{prop} \label{prop:bayes}
    When $\rvy$ is discrete and $\ell$ is cross-entropy loss, \cref{eq:one-step} is minimized for any policy $\pi$ by the Bayes classifier, or $f^*(x_s) = p(\rvy \mid x_s)$.
\end{prop}

This property requires that features are selected without knowledge of the remaining features or the response variable, which is a valid assumption for DFS, but not in scenarios where selections are based on the full feature set \citep{chen2018learning, yoon2018invase, jethani2021have}. Now, assuming that we use the Bayes classifier $f^*$ as a predictor, our second result concerns the selection policy. As we show next,
the loss-minimizing policy is equivalent to making selections based on CMI.

\vskip 0.3cm
\begin{prop} \label{prop:cmi}
    When $\rvy$ is discrete, $\ell$ is cross-entropy loss and the predictor is the Bayes classifier $f^*$, \cref{eq:one-step} is minimized by the greedy CMI policy, or $\pi^*(x_s) = \argmax_i \; I(\rvy; \rvx_i \mid x_s)$.
\end{prop}

With this, we can see that the greedy CMI policy defined in \Cref{sec:greedy} is equivalent to minimizing the one-step-ahead prediction loss. Next, we
exploit this variational perspective to develop a joint learning procedure for a policy and predictor network.

\subsection{An amortized optimization approach}

Instead of estimating each feature's CMI to identify the next selection, we now develop an approach that directly predicts the best selection at each step. The greedy policy implicitly requires solving an optimization problem for each selection, or $\argmax_i I(\rvy, \rvx_i; x_s)$, but since we lack access to this objective, we now formulate an approach that directly predicts the solution. Following a technique known as amortized optimization \citep{amos2022tutorial}, we do so by casting our variational perspective on CMI from \Cref{sec:variational} as an objective function to be optimized by a learnable network.

First, because it facilitates gradient-based optimization, we now consider that the policy outputs a \textit{distribution} over feature indices. With slight abuse of notation, this section lets the policy be a function $\pi(x_s) \in \Delta^{d-1}$, which generalizes the previous definition $\pi(x_s) \in [d]$. Using this stochastic version of the policy, we can now formulate our objective function as follows.

Let the selection policy be parameterized by a neural network $\pi(\rvx_s; \phi)$ and the predictor by a neural network $f(\rvx_s; \theta)$. Let $p(\rvs)$ represent a distribution
with support over all subsets, or
$p(s) > 0$ for all $|s| < d$. Then, our objective function $\mathcal{L}(\theta, \phi)$ is defined as
\begin{equation}
    \mathcal{L}(\theta, \phi) = \E_{p(\rvx, \rvy)} \E_{p(\rvs)} \Big[ \E_{i \sim \pi(\rvx_\rvs; \phi)} \big[ \ell\big(f(\rvx_{\rvs} \cup \rvx_i; \theta) , \rvy \big) \big] \Big]. \label{eq:objective}
\end{equation}
Intuitively, \cref{eq:objective} represents generating a random feature set $\rvx_s$, sampling a feature index according to $i \sim \pi(\rvx_s; \phi)$, and then measuring the loss of the prediction $f(\rvx_s \cup \rvx_i; \theta)$. Our objective thus optimizes for individual selections and predictions rather than the entire trajectory, which lets us build on \Cref{prop:bayes}-\ref{prop:cmi}. We describe this as an implementation of the greedy approach because it recovers the greedy CMI selections when it is trained to optimality. In the classification case, we show the following result under a mild assumption that there is a unique optimal selection.

\vskip 0.3cm
\begin{thm} \label{thm:classification}
    When $\rvy$ is discrete and $\ell$ is cross-entropy loss, the global optimum of \cref{eq:objective} is a predictor that satisfies $f(x_s; \theta^*) = p(\rvy \mid x_s)$ and a policy $\pi(x_s; \phi^*)$ that puts all probability mass on $i^* = \argmax_i I(\rvy; \rvx_i \mid x_s)$.
\end{thm}

If we relax the assumption of a unique optimal selection, the optimal policy $\pi(\rvx_s; \phi^*)$ simply splits probability mass among the best indices. A similar result holds in the regression case, where we can interpret the greedy policy as performing conditional variance minimization.

\vskip 0.3cm
\begin{thm} \label{thm:regression}
    When $\rvy$ is continuous and $\ell$ is squared error loss, the global optimum of \cref{eq:objective} is a predictor that satisfies $f(x_s; \theta^*) = \E[\rvy \mid x_s]$ and a policy $\pi(x_s; \phi^*)$ that puts all probability mass on $i^* = \argmin_i \E_{\rvx_i \mid x_s}[\Var(\rvy \mid \rvx_i, x_s)]$.
\end{thm}

Proofs for these results are in \Cref{app:proofs}.
We note that the function class for each model must be expressive enough to contain their respective optima, and that the result holds for any $p(\rvs)$ with support over all subsets.

This approach has two key advantages over the CMI estimation procedure from \Cref{sec:iterative}. First, we avoid modeling the feature conditional distributions $p(\rvx_i \mid \rvx_s)$ for all $(s, i)$. Modeling these distributions is a difficult intermediate step, and our approach instead aims to directly output the optimal index. Second, our approach is faster because each selection is made in a single forward pass: selecting $k$ features using the procedure from \citet{ma2019eddi} requires $\mathcal{O}(dk)$ scoring steps, but our approach requires only $k$ forward passes through the policy network $\pi(\rvx_s; \phi)$.

Furthermore, compared to a policy trained by RL, the greedy approach is easier to learn. Our training procedure can be viewed as a form of reward shaping \citep{sutton1998introduction, randlov1998learning}, where the reward accounts for the loss after each step and provides a strong signal about whether each selection is helpful. In comparison, observing the reward only after selecting $k$ features provides a comparably weak signal to the policy network (see \cref{eq:scoring}). RL methods generally face a challenging exploration-exploitation trade-off, but learning the greedy policy is simpler because it only requires finding the locally optimal choice
at each step.

\begin{figure*}[t]
    \centering
    \includegraphics[width=0.9\linewidth, page=11, trim={2.2cm 8.5cm 8.2cm 4.5cm}, clip]{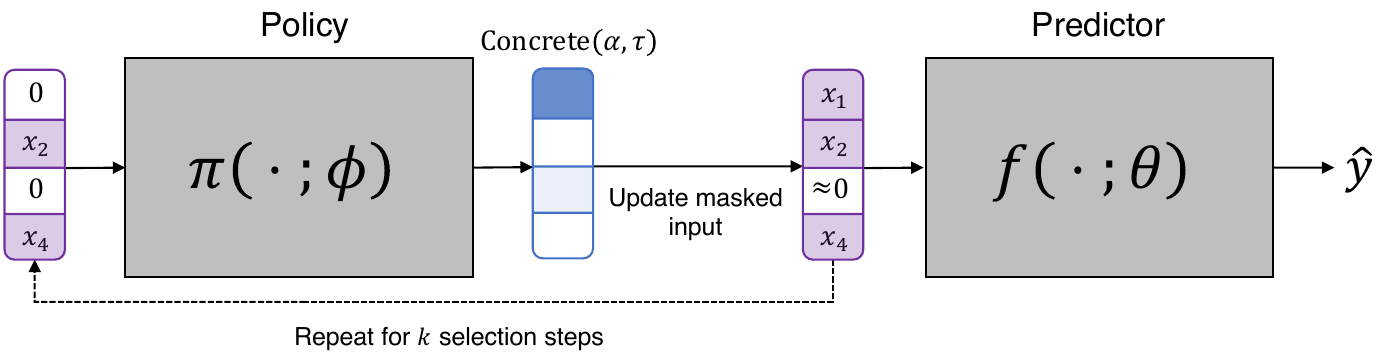}
    \vskip -0.2cm
    \caption{Diagram of our training approach. Left: features are selected by making repeated calls to the policy network using masked inputs. Right: predictions are made after each selection using the predictor network. Only solid lines are backpropagated through when performing gradient descent.} \label{fig:method}
\end{figure*}

\subsection{Training with a continuous relaxation} \label{sec:continuous}

Our objective in \cref{eq:objective} yields the correct greedy policy when it is perfectly optimized, but $\mathcal{L}(\theta, \phi)$ is difficult to optimize by gradient descent. In particular, gradients are difficult to propagate through the policy network given a sampled index $i \sim \pi(\rvx_s; \phi)$. The \textsc{Reinforce} trick \citep{williams1992simple} is one way to get stochastic gradients, but high gradient variance can make it ineffective in many problems. There is a robust literature on reducing gradient variance
in this setting
\citep{tucker2017rebar, grathwohl2018backpropagation},
but we propose using a simple alternative: the Concrete distribution \citep{maddison2016concrete}.

An index sampled according to $i \sim \pi(x_s; \phi)$ can be represented by a one-hot vector $m \in \{0, 1\}^d$ indicating the chosen index, and with the Concrete distribution we instead sample an \textit{approximately} one-hot vector in the probability simplex, or $m \in \Delta^{d-1}$. This continuous relaxation lets us calculate gradients using the reparameterization trick \citep{maddison2016concrete, jang2016categorical}. Relaxing the subset $s \subseteq [d]$ to a continuous vector also requires relaxing the policy and predictor functions, so we let these operate on a masked input $x$, or the element-wise product $x \odot m$. To avoid ambiguity about whether features are zero or masked, we can also pass the mask as a model input.

Training with the Concrete distribution requires specifying a temperature parameter $\tau > 0$ to control how discrete the samples are. Previous works have typically trained with a fixed temperature or annealed it over a pre-determined number of epochs \citep{chang2017dropout, chen2018learning, balin2019concrete}, but we instead train with a sequence of $\tau$ values and perform early stopping at each step. This removes the temperature and number of epochs as important hyperparameters to tune. Our training procedure is summarized in \Cref{fig:method}, and in more detail by \Cref{alg:method}.

There are also several optional steps that we found can improve optimization:

\begin{itemize}[leftmargin=0.5cm]
    \item Parameters can be shared between the predictor and policy networks $f(\rvx; \theta), \pi(\rvx, \phi)$. This does not complicate their joint optimization, and learning a shared representation in the early layers can in some cases help the networks optimize faster (e.g., for image data).

    \item Rather than training with a random subset distribution $p(\rvs)$, we generate subsets using features selected by the current policy $\pi(\rvx; \phi)$. This allows the models to focus on subsets likely to be encountered at inference time, and it does not affect the globally optimal policy/predictor: gradients are not propagated between selections, so both \cref{eq:objective} and this sampling approach treat each feature set as an independent optimization problem, only with different weights (see \Cref{app:training}).

    \item We pre-train the predictor $f(\rvx; \theta)$ using random subsets before jointly training the policy-predictor pair. This works better than optimizing $\mathcal{L}(\theta, \phi)$ from a random initialization, because a random predictor $f(\rvx; \theta)$ provides no signal to $\pi(\rvx; \phi)$ about which features are useful.
\end{itemize}

\section{Related work} \label{sec:related}
Prior work has frequently addressed DFS using RL. For example, \citet{dulac2011datum, shim2018joint, janisch2019classification, li2021active} optimize a reward based on the final prediction accuracy, and \citet{kachuee2018opportunistic} use a reward that accounts for prediction uncertainty. RL is a natural approach for sequential decision-making problems, but it can be difficult to optimize in practice: RL requires complex training routines, is slow to converge, and is highly sensitive to its initialization \citep{henderson2018deep}. As a result, RL-based DFS does not reliably outperform static feature selection, as shown by \citet{erion2021coai} and confirmed in our experiments.


\begin{figure*}[ht!]
    \centering
    \includegraphics[trim={0cm 2.2cm 0cm 0cm}, clip, width=0.9\linewidth]{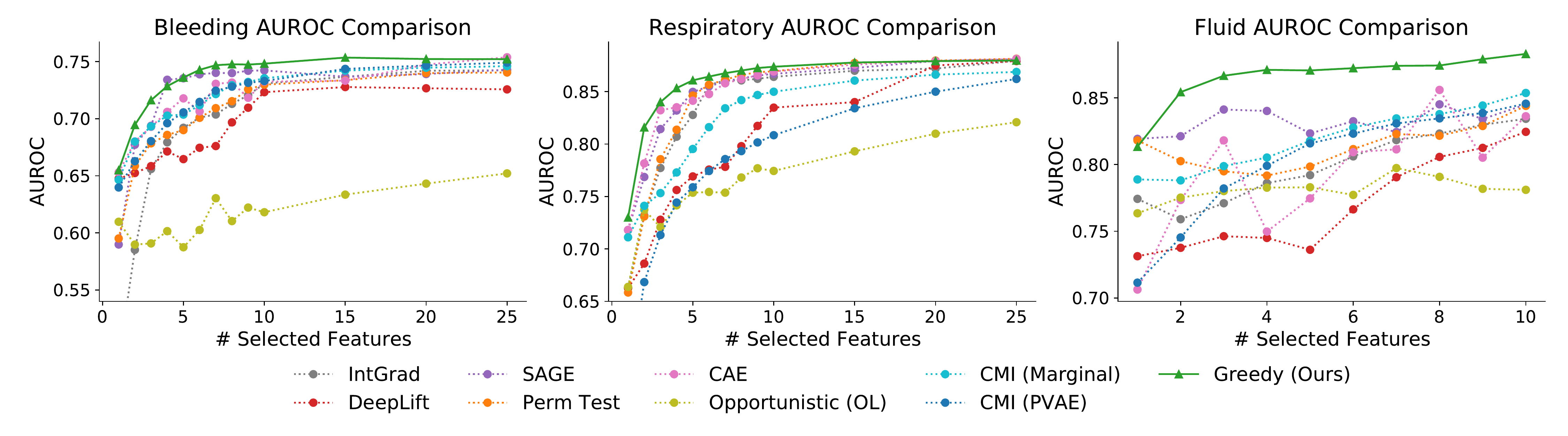}
    \includegraphics[width=0.9\linewidth]{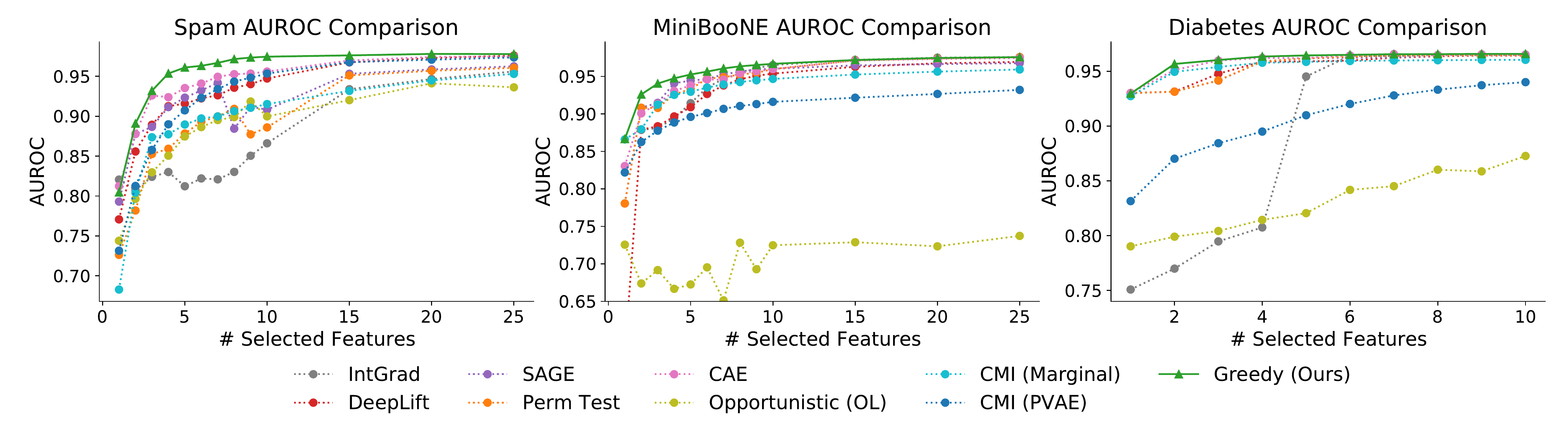}
    \vskip -0.2cm
    \caption{Evaluating the greedy approach on six tabular datasets. The results for each method are the average across five runs.} \label{fig:medical}
\end{figure*}

Several other approaches include imitation learning \citep{he2012cost, he2016active} and iterative feature scoring methods \citep{melville2004active,saar2009active, chen2015value, early2016dynamic, early2016test}. Imitation learning casts DFS as supervised classification, whereas our training approach bypasses the need for an oracle policy. Most existing feature scoring techniques are greedy methods, like ours, but they use scoring heuristics that are unrelated to maximizing CMI (see \Cref{sec:iterative}). Two feature scoring methods are specifically designed to calculate the CMI, but they suffer from important practical limitations: both \citet{ma2019eddi} and \citet{chattopadhyay2022interpretable} rely on difficult-to-train generative models, which can lead to inaccurate CMI estimation. Our approach is simpler, faster and more flexible, because the selection logic is contained within a policy network that avoids the need for generative modeling.\footnote{Concurrently, \citet{chattopadhyay2023variational} proposed a similar approach to predict the optimal selection at each step.}

Static feature selection is a long-standing problem \citep{guyon2003introduction, cai2018feature}. There are no default approaches for neural networks, but one option is ranking features by local or global importance scores \citep{breiman2001random, shrikumar2017learning, sundararajan2017axiomatic, covert2020understanding}. In addition, several prior works have leveraged continuous relaxations to learn feature selection strategies by gradient descent: for example, \citet{chang2017dropout, balin2019concrete, yamada2020feature, lee2021self, covert2022predictive} perform static feature selection, and \citet{chen2018learning, jethani2021have} perform instance-wise feature selection given access to all the features. Our work uses a similar continuous relaxation for optimization, but in the DFS context, where our method learns a selection policy rather than a static selection layer.

Finally, several works have examined greedy feature selection algorithms from a theoretical perspective. For example, \citet{das2011submodular, elenberg2018restricted} show that weak submodularity implies near-optimal performance for static feature selection. More relevant to our work, \citet{chen2015sequential} find that the related notion of adaptive submodularity \citep{golovin2011adaptive} does not not hold in the DFS setting,
but the authors
provide performance guarantees under specific distributional assumptions.

\section{Experiments} \label{sec:experiments}
We now demonstrate the use of our greedy approach
on several datasets.
We first explore tabular datasets of various sizes, including four medical diagnosis tasks, and we then consider two image classification datasets. Several of the tasks are natural candidates for DFS, and the remaining ones serve as useful tasks to test the effectiveness of our approach.
Code for reproducing our experiments is available online:
{\small \url{https://github.com/iancovert/dynamic-selection}}.

\begin{table*}[ht]
\centering
\caption{AUROC averaged across budgets of 1-10 features (with 95\% confidence intervals).} \label{tab:auroc}
\begin{center}
\vskip 0.1cm
\begin{scriptsize}
\begin{tabular}{clcccccc}
\toprule
{} & {} &  Spam &  MiniBooNE &  Diabetes &  Bleeding &  Respiratory &  Fluid \\
\midrule
\multirow{5}{*}{\rotatebox[origin=c]{90}{Static}} & IntGrad            &  82.84 $\pm$ 0.68 &  89.10 $\pm$ 0.33 &  88.91 $\pm$ 0.24 &  66.70 $\pm$ 0.27 &  81.10 $\pm$ 0.04 &  79.94 $\pm$ 0.94 \\
{} & DeepLift           &  90.16 $\pm$ 1.24 &  88.62 $\pm$ 0.30 &  95.42 $\pm$ 0.13 &  67.75 $\pm$ 0.49 &  76.05 $\pm$ 0.35 &  76.96 $\pm$ 0.56 \\
{} & SAGE               &  89.70 $\pm$ 1.10 &  92.64 $\pm$ 0.03 &  95.43 $\pm$ 0.01 &  71.34 $\pm$ 0.19 &  82.92 $\pm$ 0.26 &  83.27 $\pm$ 0.53 \\
{} & Perm Test          &  85.64 $\pm$ 3.58 &  92.19 $\pm$ 0.15 &  95.46 $\pm$ 0.02 &  68.89 $\pm$ 1.06 &  81.56 $\pm$ 0.28 &  81.35 $\pm$ 1.04 \\
{} & CAE                &  92.28 $\pm$ 0.27 &  92.76 $\pm$ 0.41 &  95.91 $\pm$ 0.07 &  70.69 $\pm$ 0.57 &  83.10 $\pm$ 0.45 &  79.40 $\pm$ 0.86 \\
\midrule
\multirow{4}{*}{\rotatebox[origin=c]{90}{Dynamic}} & Opportunistic (OL) &  85.94 $\pm$ 0.20 &  69.23 $\pm$ 0.64 &  83.07 $\pm$ 0.82 &  60.63 $\pm$ 0.55 &  74.44 $\pm$ 0.42 &  78.13 $\pm$ 0.31 \\
{} & CMI (Marginal)     &  86.57 $\pm$ 1.54 &  92.21 $\pm$ 0.40 &  95.48 $\pm$ 0.05 &  70.57 $\pm$ 0.46 &  79.62 $\pm$ 0.62 &  81.97 $\pm$ 0.93 \\
{} & CMI (PVAE)         &  89.01 $\pm$ 1.40 &  88.94 $\pm$ 1.25 &  90.50 $\pm$ 5.16 &  70.17 $\pm$ 0.74 &  74.12 $\pm$ 3.50 &  80.27 $\pm$ 1.02 \\
{} & Greedy (Ours)      &  \textbf{93.91 $\pm$ 0.17} &  \textbf{94.46 $\pm$ 0.12} &  \textbf{96.03 $\pm$ 0.02} &  \textbf{72.64 $\pm$ 0.31} &  \textbf{84.48 $\pm$ 0.08} &  \textbf{86.59 $\pm$ 0.25} \\
\bottomrule
\end{tabular}
\end{scriptsize}
\end{center}
\end{table*}

We evaluate our method by comparing to both dynamic and static feature selection methods. We also ensure consistent comparisons by only using methods applicable to neural networks.
As static baselines, we use permutation tests \citep{breiman2001random} and SAGE \citep{covert2020understanding} to rank features by their importance to the model's
accuracy, as well as per-prediction DeepLift \citep{shrikumar2017learning} and IntGrad \citep{sundararajan2017axiomatic} scores aggregated across the dataset. We then use a supervised version of the Concrete Autoencoder (CAE, \citealt{balin2019concrete}), a state-of-the-art static feature selection method. As dynamic baselines, we use two versions of the CMI estimation procedure described in \Cref{sec:iterative}. First, we use the PVAE generative model from \citet{ma2019eddi} to sample unknown features, and second, we instead sample unknown features from their marginal distribution; in both cases, we use a classifier trained with random feature subsets to make predictions. Finally, we also use the RL-based Opportunistic Learning (OL) approach \citep{kachuee2018opportunistic}. \Cref{app:baselines} provides more information about the baseline methods.

\subsection{Tabular datasets}

We first applied our method to three medical diagnosis tasks derived from an emergency medicine setting. The tasks involve predicting a patient's bleeding risk via a low fibrinogen concentration
(bleeding), whether the patient requires endotracheal intubation for respiratory support (respiratory), and whether the patient will be responsive to fluid resuscitation
(fluid). See \Cref{app:datasets} for more details about the datasets. In each scenario, gathering all possible inputs at inference time is challenging due to time and resource constraints, thus making DFS a natural solution.

We use fully connected networks for all methods, and we use dropout to reduce overfitting \citep{srivastava2014dropout}. \Cref{fig:medical} (top) shows the results of applying each method with various feature budgets. The classification accuracy is measured via AUROC, and the greedy method achieves the best results for
nearly
all feature budgets on all three tasks. Among the baselines, several static methods are sometimes close, but the CMI estimation method is rarely competitive \citep{ma2019eddi}. 
Additionally, OL provides unstable and weak results.
The greedy method's advantage is often largest when selecting a small number of features, and it usually becomes narrower once the accuracy saturates.

Next, we conducted experiments using three publicly available tabular datasets: spam classification \citep{dua2017uci}, particle identification (MiniBooNE) \citep{roc2005boosted} and diabetes diagnosis \citep{miller1973plan}. The diabetes task is a natural application for DFS and was used in prior work \citep{kachuee2018opportunistic}. We again
tested various numbers of features, and \Cref{fig:medical} (bottom)
shows plots of the AUROC for each feature budget. On these tasks,
the greedy method is once again most accurate for nearly all numbers of features. \Cref{tab:auroc} summarizes the results via the mean AUROC across $k = 1, \ldots, 10$ features, further emphasizing the benefits of the greedy method across all six datasets. \Cref{app:results} shows larger versions of the AUROC curves (\Cref{fig:medical-large} and \Cref{fig:tabular-large}),
as well as plots demonstrating the variability of selections within each dataset.

The results with these datasets reveal that, perhaps surprisingly, dynamic methods can
be outperformed by static methods. Interestingly, this point was not highlighted in prior works where strong static baselines were not tested \citep{kachuee2018opportunistic, janisch2019classification}. For example, OL is not competitive on these datasets, and the two versions of the CMI estimation approach are not consistently among the top baselines.
Dynamic methods are in principle capable of performing better, so the sub-par results from these methods underscore the difficulty of learning both a selection policy and a prediction function that works for multiple feature sets. In these experiments, our approach is the only dynamic method to do both successfully.

\begin{figure*}[t]
    \centering
    \includegraphics[width=0.9\linewidth]{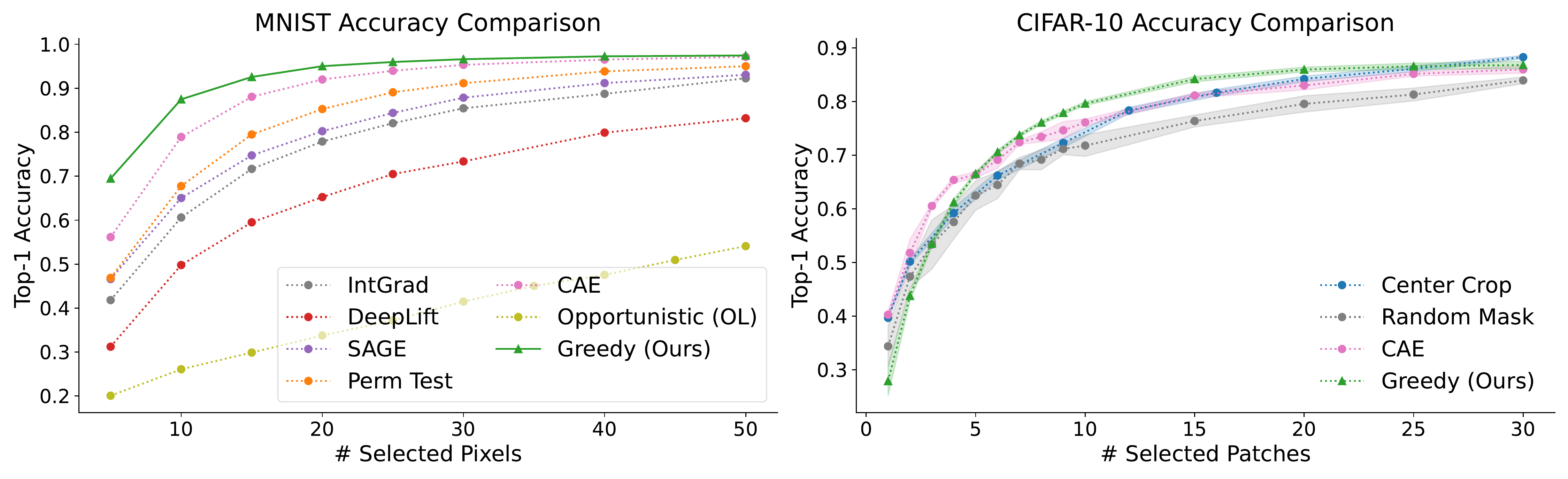}
    \includegraphics[width=0.9\linewidth]{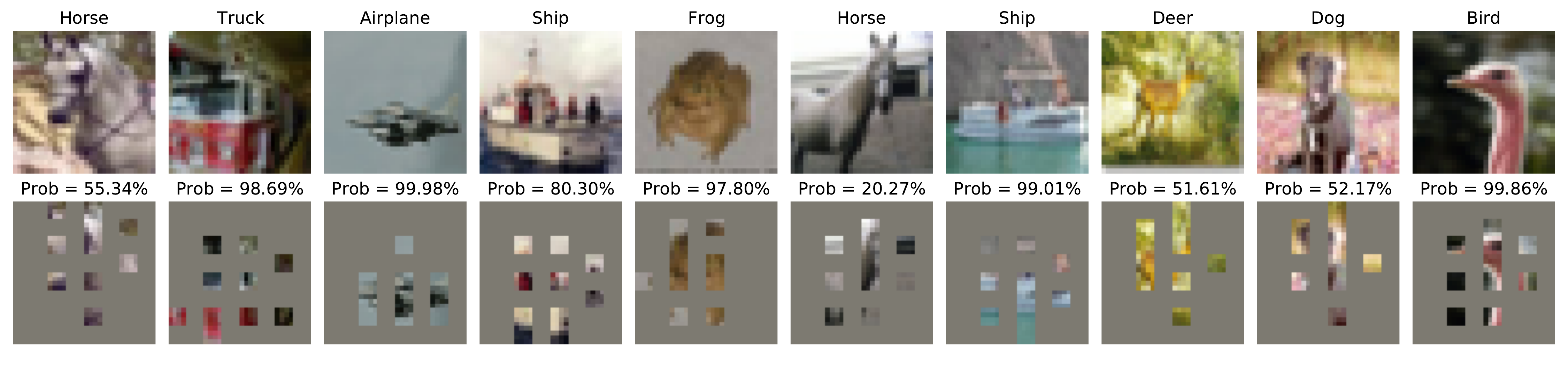}
    \caption{Greedy feature selection for image classification. Top left: accuracy comparison on MNIST with results averaged across five runs. Top right: accuracy comparison on CIFAR-10 with 95\% confidence intervals. Bottom: example selections and predictions for the greedy method with 10 out of 64 patches for CIFAR-10 images.} \label{fig:images}
\end{figure*}

\subsection{Image classification datasets}

Next, we considered two standard image classification datasets: MNIST \citep{lecun1998gradient} and CIFAR-10 \citep{krizhevsky2009learning}. Our goal is to begin with a blank image, sequentially reveal multiple pixels or patches, and ultimately make a classification using a small portion of the image. Although this is not an obvious use case for DFS, it represents a challenging problem for our method, and similar tasks were considered in several earlier works \citep{karayev2012timely, mnih2014recurrent, early2016test, janisch2019classification}.

For MNIST, we use fully connected architectures for both the policy and predictor,
and we treat pixels as individual features; we therefore have $d = 784$.
For CIFAR-10, we use a shared ResNet backbone \citep{he2016deep} for the policy and predictor networks, and each network uses its own output head. The $32 \times 32$ images are coarsened into $d = 64$ patches of size $4 \times 4$, so the selector head generates
logits corresponding to each patch, and the predictor head generates probabilities for each class.

\Cref{fig:images} shows our method's accuracy for different feature budgets.
For MNIST, we use the previous baselines but exclude the CMI estimation method due to its computational cost: it becomes slow when evaluating many candidate features. We observe a large benefit for our method,
particularly when making a small number of selections. Our greedy method reaches nearly 90\% accuracy with just 10 pixels, which is roughly 10\% higher than the best baseline and considerably higher than prior work \citep{balin2019concrete, yamada2020feature, covert2020understanding}. OL yields the worst results, and it also trains slowly due to the large number of states.

For CIFAR-10, we omit several baseline comparisons due to their computational cost. We use the CAE, which is our most competitive static baseline, as well as two simple baselines: center crops and random masks of various sizes.
For each method, we plot the mean and 95\% confidence intervals determined from five trials.
Our greedy approach is slightly less accurate with a very small number of patches, but it reaches significantly higher accuracy when using 6-20 patches. Finally, \Cref{fig:images} (bottom) also shows qualitative examples of our method's predictions after selecting 10 out of 64 patches, and \Cref{app:results} shows similar plots with different numbers of patches.

\section{Conclusion} \label{sec:conclusion}
In this work, we explored a greedy algorithm for dynamic feature selection (DFS) that selects features based on their CMI with the response variable. We proposed an approach to approximate this policy by directly predicting the optimal selection at each step, and we conducted experiments that show our method outperforms a variety of existing feature selection methods, including both dynamic and static baselines.
Future work on this topic may include incorporating non-uniform features costs or determining the ideal feature budget on a per-sample basis; from a theoretical perspective, characterizing the greedy algorithm's performance outside of our fixed-budget case is another interesting topic for future work \cite{chen2015sequential}.
Finally, future work may also explore architectures that are well-suited to processing partial inputs, particularly for structured data like images.

\section*{Acknowledgements}
We thank Samuel Ainsworth, Kevin Jamieson, Mukund Sudarshan and the Lee Lab for helpful discussions. This work was funded by NSF DBI-1552309 and DBI-1759487, NIH R35-GM-128638 and R01-NIA-AG-061132.

\bibliography{main}
\bibliographystyle{icml2023}

\newpage
\appendix
\onecolumn
\clearpage
\section{Proofs} \label{app:proofs}

In this section, we re-state and prove our main theoretical results. We begin with our proposition regarding the optimal predictor for an arbitrary policy $\pi$.

\vskip 0.3cm
\begin{appprop}
    When $\rvy$ is discrete and $\ell$ is cross-entropy loss, \cref{eq:one-step} is minimized for any policy $\pi$ by the Bayes classifier, or $f^*(x_s) = p(\rvy \mid x_s)$.
\end{appprop}

\begin{proof}
    Given the predictor inputs $x_s$, our goal is to determine the prediction that minimizes the expected loss. Because features are selected sequentially by $\pi$ with no knowledge of the non-selected values, there is no other information to condition on; for the predictor, we do not even need to distinguish the order in which features were selected. We can therefore derive the optimal prediction $\hat y \in \Delta^{K -  1}$ for a discrete response $\rvy \in [K]$ as follows:
    
    \begin{align*}
        f^*(x_s) &= \argmin_{\hat y} \; \E_{\rvy \mid x_s} \big[ \ell(\hat y, \rvy) \big] \\
        &= \argmin_{\hat y} \; \sum_{i \in \gY} p(\rvy = i \mid x_s) \log \hat{y}_i \\
        &= \argmin_{\hat y} \; \KL\big( p(\rvy \mid x_s) \mid\mid \hat y \big) + H(\rvy \mid x_s) \\
        &= p(\rvy \mid x_s).
    \end{align*}
    
    In the case of a continuous response $\rvy \in \R$ with squared error loss, we have a similar result involving the response's conditional expectation:

    \begin{align*}
        f^*(x_s) &= \argmin_{\hat y} \; \E_{\rvy \mid x_s} \big[ (\hat y - \rvy)^2 \big] \\
        &= \argmin_{\hat y} \; \E_{\rvy \mid x_s} \big[ (\hat y - \E[\rvy \mid x_s])^2 \big] + \Var(\rvy \mid x_s) \\
        &= \E[\rvy \mid x_s].
    \end{align*}
\end{proof}

\begin{appprop}
    When $\rvy$ is discrete, $\ell$ is cross-entropy loss and the predictor is the Bayes classifier $f^*$, \cref{eq:one-step} is minimized by the greedy CMI policy, or $\pi^*(x_s) = \argmax_i \; I(\rvy; \rvx_i \mid x_s)$.
\end{appprop}

\begin{proof}
    Following \cref{eq:one-step}, the policy network's selection $i = \pi(x_s)$ incurs the following expected loss with respect to the distribution $p(\rvy, \rvx_i \mid x_s)$:

    \begin{align*}
        \E_{\rvy, \rvx_i \mid x_s} \big[ \ell(f^*(x_s \cup \rvx_i), \rvy) \big]
        &= \E_{\rvy, \rvx_i \mid x_s} \big[ \ell(p(\rvy \mid \rvx_i, x_s), \rvy) \big] \\
        &= \E_{\rvx_i \mid x_s} \Big[ \E_{\rvy \mid \rvx_i, x_s} [\ell(p(\rvy \mid \rvx_i, x_s), \rvy)] \Big] \\
        &= \E_{\rvx_i \mid x_s} \big[ H(\rvy \mid \rvx_i, x_s) \big] \\
        &= H(\rvy \mid x_s) - I(\rvy; \rvx_i \mid x_s).
    \end{align*}

    Note that $H(\rvy \mid x_s)$ is a constant that does not depend on $i$. When identifying the index that minimizes the
    expected loss, we therefore have the following result:

    \begin{align*}
        \argmin_i \; \E_{\rvy, \rvx_i \mid x_s} \big[ \ell(f^*(x_s \cup \rvx_i), \rvy) \big] = \argmax_i \; I(\rvy; \rvx_i \mid x_s).
    \end{align*}

    In the case of a continuous response with squared error loss and an optimal predictor given by $f^*(x_s) = \E[\rvy \mid x_s]$, we have a similar result:

    \begin{align*}
        \E_{\rvy, \rvx_i \mid x_s} \big[ (f^*(x_s \cup \rvx_i) - \rvy)^2 \big]
        &= \E_{\rvy, \rvx_i \mid x_s} \big[ (\E[\rvy \mid \rvx_i, x_s] - \rvy)^2 \big] \\
        &= \E_{\rvx_i \mid x_s} \Big[ \E_{\rvy \mid \rvx_i, x_s} [(\E[\rvy \mid \rvx_i, x_s] - \rvy)^2] \Big] \\
        &= \E_{\rvx_i \mid x_s} [\Var(\rvy \mid \rvx_i, x_s)].
    \end{align*}

    When we aim to minimize the expected loss, our selection is therefore the index that yields the lowest expected conditional variance:

    \begin{align*}
        \argmin_i \; \E_{\rvx_i \mid x_s} [\Var(\rvy \mid \rvx_i, x_s)].
    \end{align*}
\end{proof}

Next, we also prove the limiting result presented in \cref{eq:limiting}, which states that $I_i^n \to I(\rvy; \rvx_i \mid x_s)$.

\begin{proof}
    The conditional mutual information $I(\rvy; \rvx_i \mid x_s)$ is defined as follows \citep{cover2012elements}:

    \begin{align*}
        I(\rvy; \rvx_i \mid x_s) &= \KL \big( p(\rvx_i, \rvy \mid x_s) \mid\mid p(\rvx_i \mid x_s) p(\rvy \mid x_s) \big) \\
        &= \E_{\rvy, \rvx_i \mid x_s} \Big[ \log \frac{p(\rvy, \rvx_i \mid x_s)}{p(\rvx_i \mid x_s) p(\rvy \mid x_s)} \Big].
    \end{align*}

    Rearranging terms, we can write this as an expected KL divergence with respect to $\rvx_i$:

    \begin{align*}
        I(\rvy; \rvx_i \mid x_s) &= \E_{\rvx_i \mid x_s} \E_{\rvy \mid x_s, \rvx_i} \Big[ \log \frac{p(\rvy, \rvx_i \mid x_s)}{p(\rvx_i \mid x_s) p(\rvy \mid x_s)} \Big] \\
        &= \E_{\rvx_i \mid x_s} \E_{\rvy \mid x_s, \rvx_i} \Big[ \log \frac{p(\rvy \mid \rvx_i, x_s)}{p(\rvy \mid x_s)} \Big] \\
        &= \E_{\rvx_i \mid x_s} \Big[ \KL\big( p(\rvy \mid \rvx_i, x_s) \mid\mid p(\rvy \mid x_s)\big) \Big]
    \end{align*}

    Now, when we sample multiple values $x_i^1, \ldots, x_i^n \sim p(\rvx_i \mid x_s)$ and make predictions using the Bayes classifier, we have the following mean prediction as $n$ becomes large:

    \begin{equation*}
        \lim_{n \to \infty} \; \frac{1}{n} \sum_{j = 1}^n p(\rvy \mid x_s, x_i^j) = \E_{\rvx_i \mid x_s} \big[ p(\rvy \mid \rvx_i, x_s) \big] = p(\rvy \mid x_s).
    \end{equation*}

    Calculating the mean KL divergence relative to this prediction, we arrive at the following result:

    \begin{align*}
        \lim_{n \to \infty} \; I_i^n &= \E_{\rvx_i \mid x_s} \Big[ \KL\big( p(\rvy \mid \rvx_i, x_s) \mid\mid p(\rvy \mid x_s)\big) \Big] = I(\rvy ; \rvx_i \mid x_s).
    \end{align*}
\end{proof}

\begin{appthm}
    When $\rvy$ is discrete and $\ell$ is cross-entropy loss, the global optimum of \cref{eq:objective} is a predictor that satisfies $f(x_s; \theta^*) = p(\rvy \mid x_s)$ and a policy $\pi(x_s; \phi^*)$ that puts all probability mass on $i^* = \argmax_i I(\rvy; \rvx_i \mid x_s)$.
\end{appthm}

\begin{proof}
    We first consider the predictor network $f(\rvx_s; \theta)$. When the predictor is given the feature values $x_s$, it means that one index $i \in s$ was chosen by the policy according to $\pi(x_{s \setminus i}; \phi)$ and the remaining indices $s \setminus i$ were sampled from $p(\rvs)$. Because $\rvs$ is sampled independently from $(\rvx, \rvy)$, and because $\pi(x_{s \setminus i}; \phi)$ is not given access to $(\rvx_{[d] \setminus s}, \rvx_i, \rvy)$, the predictor's expected loss must be considered with respect to the distribution $\rvy \mid x_s$. The globally optimal predictor $f(x_s; \theta^*)$ is thus defined as follows, regardless of the selection policy $\pi(x_s; \phi)$ and which index $i$ was selected last:

    \begin{align*}
        f(x_s; \theta^*) &= \argmin_{\hat y} \; \E_{\rvy \mid x_s}\big[ \ell(\hat y, \rvy) \big] = p(\rvy \mid x_s).
    \end{align*}

    The above result follows from our proof for \Cref{prop:bayes}. Now, given the optimal predictor $f(x_s; \theta^*)$, we can define the globally optimal policy by minimizing the expected loss for a fixed input $x_s$. Denoting the probability mass placed on each index $i \in [d]$ as $\pi_i(x_s; \phi)$, where $\pi(x_s; \phi) \in \Delta^{d-1}$, the expected loss is the following:

    \begin{align*}
        \E_{i \sim \pi(x_s; \phi)} \E_{\rvy, \rvx_i \mid x_s} \big[ \ell(f(x_s \cup \rvx_i; \theta^*), \rvy) \big]
        &= \sum_{i \in [d]} \pi_i(x_s; \phi) \E_{\rvy, \rvx_i \mid x_s} \big[ \ell\big(f(x_s \cup \rvx_i; \theta^*), \rvy\big) \big] \\
        &= \sum_{i \in [d]} \pi_i(x_s; \phi) \E_{\rvx_i \mid x_s} [H(\rvy \mid \rvx_i, x_s)].
    \end{align*}

    The above result follows from our proof for \Cref{prop:cmi}. If there exists a single index $i^* \in [d]$ that yields the lowest expected conditional entropy, or

    \begin{align*}
        \E_{\rvx_{i^*} \mid x_s} [H(\rvy \mid \rvx_{i^*}, x_s)] < \E_{\rvx_i \mid x_s} [H(\rvy \mid \rvx_i, x_s)] \quad \forall i \neq i^*,
    \end{align*}

    then the optimal predictor must put all its probability mass on $i^*$, or $\pi_{i^*}(x_s; \phi^*) = 1$. Note that the corresponding feature $\rvx_{i^*}$ has maximum conditional mutual information with $\rvy$, because we have

    \begin{align*}
        I(\rvy; \rvx_{i^*} \mid x_s) = \underbrace{H(\rvy \mid x_s)}_{\mathrm{Constant}} - \E_{\rvx_{i^*} \mid x_s} [H(\rvy \mid \rvx_{i^*}, x_s)].
    \end{align*}

    To summarize, we derived the global optimum to our objective $\mathcal{L}(\theta, \phi)$ by first considering the optimal predictor $f(\rvx_s; \theta^*)$, and then considering the optimal policy $\pi(\rvx_s; \phi^*)$ when we assume that we use the optimal predictor.
\end{proof}

\begin{appthm}
    When $\rvy$ is continuous and $\ell$ is squared error loss, the global optimum of \cref{eq:objective} is a predictor that satisfies $f(x_s; \theta^*) = \E[\rvy \mid x_s]$ and a policy $\pi(x_s; \phi^*)$ that puts all probability mass on $i^* = \argmin_i \E_{\rvx_i \mid x_s}[\Var(\rvy \mid \rvx_i, x_s)]$.
\end{appthm}

\begin{proof}
    Our proof follows the same logic as our proof for \Cref{thm:classification}. For the optimal predictor given an arbitrary policy, we have:

    \begin{align*}
        f(x_s; \theta^*) = \argmin_{\hat y} \; \E_{\rvy \mid x_s} \big[ (\hat y - \rvy)^2 \big] = \E[\rvy \mid x_s].
    \end{align*}

    Then, for the policy's expected loss, we have:

    \begin{align*}
        \E_{i \sim \pi(x_s; \phi)} \E_{\rvy, \rvx_i \mid x_s} \big[ \big(f(x_s \cup \rvx_i; \theta^*) - \rvy\big)^2 \big]
        &= \sum_{i \in [d]} \pi_i(x_s; \phi) \E_{\rvx_i \mid x_s} [\Var(\rvy \mid \rvx_i, x_s)].
    \end{align*}

    If there exists an index $i^* \in [d]$ that yields the lowest expected conditional variance, then the optimal policy must put all its probability mass on $i^*$, or $\pi_{i^*}(x_s; \phi^*) = 1$.
\end{proof}

\clearpage
\section{Datasets} \label{app:datasets}
The datasets used in our experiments are summarized in \Cref{tab:datasets}. Three of the tabular datasets and the two image classification datasets are publicly available, and the three emergency medicine tasks were privately curated from
the Harborview Medical Center Trauma Registry.

\begin{table}[h!]
    \caption{Summary of datasets used in our experiments.} \label{tab:datasets}
    \vskip 0.1cm
    \begin{center}
    \begin{tabular}
    {lcccc}
    \toprule
    Dataset & \# Features & \# Feature Groups & \# Classes & \# Samples \\
    \midrule
    Fluid & 224 & 162 & 2 & 2,770 \\
    Respiratory & 112 & 35 & 2 &  65,515 \\
    Bleeding & 121 & 44 & 2 & 6,496\\
    \midrule
    Spam  &  58 & -- & 2 & 4,601  \\ 
    MiniBooNE  & 51 & -- & 2 & 130,064 \\ 
    Diabetes & 45 & -- & 3 & 92,062 \\
    \midrule
    MNIST  & 784 & -- & 10 & 60,000\\
    CIFAR-10 & 1,024 & 64 & 10 & 60,000 \\
    \bottomrule
    \end{tabular}
    \end{center}
\end{table}

\subsection{MiniBooNE and spam classification}

The spam dataset includes features extracted from e-mail messages to predict whether or not a message is spam. Three features describes the usage of capital letters in the e-mail, and the remaining 54 features describe the frequency with which certain key words or characters are used. The MiniBooNE particle identification dataset involves distinguishing electron neutrinos from muon neutrinos based on various continuous features \citep{roc2005boosted}. Both datasets were obtained from the UCI repository \citep{dua2017uci}.

\subsection{Diabetes classification}

The diabetes dataset was obtained from from the National Health and Nutrition Examination Survey (NHANES) \citep{NHANES}, an ongoing survey designed to assess the well-being of adults and children in the United States. We used a version of the data pre-processed by \citet{kachuee2018opportunistic, kachuee2019cost} that includes data collected from 1999 through 2016. The input features include demographic information (age, gender, ethnicity, etc.), lab results (total cholesterol, triglyceride, etc.), examination data (weight, height, etc.), and questionnaire answers (smoking, alcohol, sleep habits, etc.). An expert was also asked to suggest costs for each feature based on the financial burden, patient privacy, and patient inconvenience, but we assume uniform feature costs in our experiments. Finally, the fasting glucose values were used to define three classes based on standard threshold values: normal, pre-diabetes, and diabetes.

\subsection{Image classification datasets}

The MNIST and CIFAR-10 datasets were downloaded using PyTorch \citep{paszke2017automatic}. We used the standard train-test splits, and we split the train set to obtain a validation set with the same size as the test set (10,000 examples).

\subsection{Emergency medicine datasets}

The emergency medicine datasets used in this study were gathered over a 13-year period (2007-2020) and encompass 14,463 emergency department admissions. We excluded patients under the age of 18, and we curated 3 clinical cohorts commonly seen in pre-hospitalization settings. These include 1)~pre-hospital fluid resuscitation, 2)~emergency department respiratory support, and 3)~bleeding after injury. These datasets are not publicly available due to patient privacy concerns.

\paragraph{Pre-hospital fluid resuscitation}
We selected 224 variables that were available in the pre-hospital setting, including dispatch information (injury date, time, cause, and location), demographic information (age, sex), and pre-hospital vital signs (blood pressure, heart rate, respiratory rate). The outcome was each patient's response to fluid resuscitation, following the Advanced Trauma Life Support (ATLS) definition \citep{subcommittee2013advanced}. 

\paragraph{Emergency department respiratory support}
In this cohort, our goal is to predict which patients require respiratory support upon arrival in the emergency department. Similar to the previous dataset, we selected 112 pre-hospital clinical features including dispatch information (injury date, time, cause, and location), demographic information (age, sex), and pre-hospital vital signs (blood pressure, heart rate, respiratory rate). The outcome is defined based on whether a patient received respiratory support, including both invasive (intubation) and non-invasive (BiPap) approaches.

\paragraph{Bleeding}
In this cohort, we only included patients whose fibrinogen levels were measured, as this provides an indicator for bleeding or fibrinolysis \citep{mosesson2005fibrinogen}. As with the previous datasets, demographic information, dispatch information, and pre-hospital observations were used as input features. The outcome, based on experts' opinion, was defined by whether an individual's fibrinogen level is below 200 mg/dL, which represents higher risk of bleeding after injury.

\section{Baselines} \label{app:baselines}

This section provides more details on the baseline methods used in our experiments (\Cref{sec:experiments}).

\subsection{Global feature importance methods}

Two of our static feature selection baselines, permutation tests and SAGE, are \textit{global feature importance methods} that rank features based on their role in improving model accuracy \citep{covert2021explaining}. In our experiments, we ran each method using a single classifier trained on the entire dataset, and we then selected the top $k$ features depending on the budget.

When running the permutation test, we calculated the validation AUROC while replacing values in the corresponding feature column with random draws from the training set. When running SAGE, we used the authors' implementation with automatic convergence detection \citep{covert2020understanding}. To handle held-out features, we averaged across 128 sampled values for the six tabular datasets, and for MNIST we used a zeros baseline to achieve faster convergence.

\subsection{Local feature importance methods}

Two of our static feature selection baselines, DeepLift and Integrated Gradients, are \textit{local feature importance methods} that rank features based on their importance to a single prediction. In our experiments, we generated feature importance scores for the true class using all examples in the validation set. We then selected the top $k$ features based on their mean absolute importance. We used a mean baseline for Integrated Gradients \citep{sundararajan2017axiomatic}, and both methods were run using the Captum package \citep{kokhlikyan2020captum}.

\subsection{Differentiable feature selection}

Our last static feature selection baseline is the Concrete autoencoder (CAE) from \citet{balin2019concrete}. The method was originally proposed to perform unsupervised feature selection by reconstructing the full input vector, but we changed the prediction target to use it in a supervised fashion. The authors propose training with an exponentially decayed temperature over a hand-tuned number of epochs, but we used an approach similar to our own method: we trained with a sequence of temperature values, performing early stopping using the validation loss for each one, and we returned the features chosen after training with the final temperature.

We tried a similar method proposed by \citet{yamada2020feature}, but this method requires tuning a penalty parameter to achieve the desired number of features, and we found that it gave similar performance in our experiments on MNIST. Among methods that learn to select features within a neural network, there are several others that do so using group sparse penalties \cite{feng2017sparse, tank2021neural, lemhadri2021lassonet}; we tested the LassoNet approach from \citet{lemhadri2021lassonet} and found that it was not competitive on MNIST. For simplicity, we present results only for the supervised CAE.

\subsection{CMI estimation}

Our experiments use two versions of the CMI estimation approach described in \Cref{sec:iterative}. Both are inspired by the EDDI method introduced by \citet{ma2019eddi}, but a key difference is that we do not jointly model $(\rvx, \rvy)$ within the same conditional generative model: we instead separately model the response with a classifier $f(\rvx_s) \approx p(\rvy \mid \rvx_s)$ and the features with a generative model of $p(\rvx_i \mid \rvx_s)$. This partially mitigates one challenge with this approach, which is working with mixed continuous/categorical data (i.e., we do not need to jointly model categorical response variables).

For the first version of this approach, we train a PVAE as a generative model \citep{ma2019eddi}. The encoder and decoder both have two hidden layers, the latent dimension is set to 16, and we use 128 samples from the latent posterior to approximate $p(\rvx_i \mid x_s) = \int p(\rvx_i \mid \rvz) p(\rvz \mid x_s)$. We use Gaussian distributions for both the latent and decoder spaces, and we generate samples using the decoder mean, similar to the original approach \citep{ma2019eddi}. In the second version, we bypass the need for a generative model with a simple approximation: we sample features from their marginal distribution, which is equivalent to assuming feature independence.

\subsection{Opportunistic learning}

\citet{kachuee2018opportunistic} proposed Opportunistic Learning (OL), an approach to solve DFS using RL. The model consists of two networks analogous to our policy and predictor: a Q-network that estimates the value associated with each action, where actions correspond to features, and a P-network responsible for making predictions. When using OL, we use the same architectures as our approach, and OL shares network parameters between the P- and Q-networks.

The authors introduce a utility function for their reward, shown in \cref{app:OL_reward}, which calculates the difference in prediction uncertainty as approximated by MC dropout \citep{gal2016dropout}. The reward also accounts for feature costs, but we set all feature costs to $c_i = 1$:

\begin{equation}\label{app:OL_reward}
    r_i = \frac{|| Cert(x_s) - Cert(x_s \cup x_i)||}{c_i}
\end{equation}

To provide a fair comparison with the remaining methods, we made several modifications to the authors' implementation. These include 1)~preventing the prediction action until the pre-specified budget is met, 2)~setting all feature costs to be identical, and 3)~supporting pre-defined feature groups as described in \Cref{app:feature_grouping}. When training, we update the P-, Q-, and target Q-networks every $1 + \frac{d}{100}$ experiences, where $d$ is the number of features in a dataset. In addition, the replay buffer is set to store the $1000 d$ most recent experiences, and the random exploration probability is decayed so that it eventually reaches a value of 0.1.

\clearpage
\section{Training approach and hyperparameters} \label{app:training}

This section provides more details on our training approach and hyperparameter choices.

\subsection{Training pseudocode}

\Cref{alg:method} summarizes our training approach. Briefly, we select features by drawing a Concrete sample using policy network's logits, we calculate the loss based on the subsequent prediction, and we then update the mask for the next step using a discrete sample from the policy's distribution. We implemented this approach using PyTorch \citep{paszke2017automatic} and PyTorch Lightning.\footnote{\url{https://www.pytorchlightning.ai}}

\begin{algorithm}[H]
  \SetAlgoLined
  \DontPrintSemicolon
  \KwInput{Data distribution $p(\rvx, \rvy)$, budget $k > 0$, learning rate $\gamma > 0$, temperature $\tau > 0$}
  \KwOutput{Predictor model $f(\rvx ; \theta)$, policy model $\pi(\rvx ; \phi)$}
  initialize $f(\rvx; \theta), \pi(\rvx; \phi)$ \;
  \While{not converged}{
      sample $x, y \sim p(\rvx, \rvy)$ \;
      initialize $\mathcal{L} = 0$, $m = [0, \ldots, 0]$ \;
      \For{$j = 1$ \KwTo $k$}{
          calculate logits $\alpha = \pi(x \odot m; \phi)$, sample $G_i \sim \text{Gumbel}$ for $i \in [d]$ \;
          set $\tilde{m} = \max\big( m, \softmax(G + \alpha, \tau) \big)$ \tcp{update with Concrete} 
          set $m = \max \big(m, \softmax(G + \alpha, 0) \big)$ \tcp{update with one-hot} 
          update $\mathcal{L} \gets \mathcal{L} + \ell\big(f(x \odot \tilde{m} ; \theta), y \big)$ \;
      }
      update $\theta \gets \theta - \gamma \nabla_\theta \mathcal{L}, \;\; \phi \gets \phi - \gamma \nabla_\phi \mathcal{L}$ \;
  }
  \Return{$f(\rvx; \theta), \pi(\rvx; \phi)$}
\caption{Training pseudocode}
\label{alg:method}
\end{algorithm}

One notable difference between \Cref{alg:method} and our objective $\mathcal{L}(\theta, \phi)$ in the main text is the use of the policy $\pi(\rvx; \phi)$ for generating feature subsets. This differs from \cref{eq:objective}, which generates feature subsets using a subset distribution $p(\rvs)$. The key shared factor between both approaches is that there are separate optimization problems over each feature set that are effectively treated independently. For each feature set $x_s$, the problem is the one-step-ahead loss, and it incorporates both the policy and predictor as follows:
\begin{equation}
    \E_{i \sim \pi(\rvx_\rvs; \phi)} \big[ \ell\big(f(\rvx_{\rvs} \cup \rvx_i; \theta) , \rvy \big) \big].
\end{equation}
The problems for each subset do not interact: during optimization, the selection given $x_s$ is based only on the immediate change in the loss, and gradients are not propagated through multiple selections as they would be for an RL-based solution. In solving these multiple problems, the difference is simply that \cref{eq:objective} weights them according to $p(\rvs)$, whereas \Cref{alg:method} weights them according to the current policy $\pi(\rvx, \phi)$.

We find that incorporating the current policy when generating feature sets is important to achieve good performance. As an ablation, we tested how much our method's performance changes when we instead generate training examples $(\rvx_\rvs, \rvy)$ at random rather than using the current policy: using the MNIST dataset,
we find that using random subsets leads to a significant drop in performance (\Cref{tab:ablation}).

\subsection{Model selection}

One detail not shown in \Cref{alg:method} that we alluded to in the main text is our approach for decaying the Concrete distribution's temperature parameter $\tau$. We train with a sequence of relatively few temperature values, using the validation loss to perform early stopping with each value. To perform model selection, we separately calculate the validation loss using a temperature value of zero, which more accurately represents the model's usage at inference time; we eventually return the version of the model that performed best on this zero-temperature loss, chosen across all training temperatures.

\begin{table*}[t]
\centering
\caption{Ablation experiment using MNIST.} \label{tab:ablation}
\begin{center}
\vskip 0.1cm
\begin{small}
\begin{tabular}{ccccccccc}
\toprule
\# Features & 5 & 10 & 15 & 20 & 25 & 30 & 40 & 50 \\
\midrule
Ours & 0.695 & 0.875 & 0.926 & 0.950 & 0.960 & 0.966 & 0.973 & 0.975 \\
Ablation & 0.578 & 0.757 & 0.807 & 0.819 & 0.838 & 0.850 & 0.869 & 0.883 \\
\bottomrule
\end{tabular}
\end{small}
\end{center}
\end{table*}

\subsection{Hyperparameters}

Our experiments with the six tabular datasets used fully connected architectures with dropout in all layers \citep{srivastava2014dropout}. The dropout probability is set to 0.3, the networks have two hidden layers of width 128, and we performed early stopping using the validation loss. For our method, the predictor and policy were separate networks with identical architectures. When training models with the features selected by static methods, we reported results using the best model from multiple training runs based on the validation loss. We did not perform any additional hyperparameter tuning due to the large number of models being trained.

For MNIST, we used fully connected architectures with two layers of width 512 and the dropout probability set to 0.3. Again, our method used separate networks with identical architectures. For CIFAR-10, we used a shared ResNet backbone \citep{he2016deep} consisting of several residually connected convolutional layers. The classification head consists of global average pooling and a linear layer, and the selection head consisted of a transposed convolution layer followed by a $1 \times 1$ convolution, which outputs a grid of logits with size $8 \times 8$. Our CIFAR-10 networks are trained using random crops and random horizontal flips as augmentations.

\subsection{Feature grouping}\label{app:feature_grouping}

All of the methods used in our experiments were designed to select individual features, but this is undesirable when using categorical features with one-hot encodings. Each of our three emergency medicine tasks involve such features, so we extended each method to support feature grouping.

SAGE and permutation tests are trivial to extend to feature groups: we simply removed groups of features rather than individual features when calculating importance scores. For DeepLift and Integrated Gradients, we used the summed importance within each group, which preserves each method's additivity property. For the method based on Concrete Autoencoders, we implemented a generalized version of the selection layer that operates on feature groups. We also extended OL to operate on feature groups by having actions map to groups rather than individual features.

Finally, for our method, we parameterized the policy network $\pi(\rvx; \phi)$ so that the number of outputs is the number of groups $g$ rather than the total number of features $d$ (where $g < d$). When applying masking, we first generate a binary mask $m \in [0, 1]^g$, and we then project the mask into $[0, 1]^d$ using a binary group matrix $G \in \{0, 1\}^{d \times g}$, where $G_{ij} = 1$ if feature $i$ is in group $j$ and $G_{ij} = 0$ otherwise. Thus, our masked input vector is given by $x \odot (Gm)$.

\section{Additional results} \label{app:results}

This section provides several additional experimental results. First, \Cref{fig:medical-large} and \Cref{fig:tabular-large} show the same results as \Cref{fig:medical} but larger for improved visibility. Next, \Cref{fig:fib-freqency} though \Cref{fig:diabetes-freqency} display the feature selection frequency for each of the tabular datasets when using the greedy method. The heatmaps in each plot show the portion of the time that a feature (or feature group) is selected under a specific feature budget. These plots reveal that our method is indeed selecting different features for different samples.

Finally, \Cref{fig:cifar-large} displays examples of CIFAR-10 predictions given different numbers of revealed patches. The predictions generally become relatively accurate after revealing only a small number of patches, reflecting a similar result as \Cref{fig:images}. Qualitatively, we can see that the policy network learns to select vertical stripes, but the order in which it fills out each stripe depends on where it predicts important information may be located.

\begin{figure}[ht]
    \centering
    \includegraphics[width=0.7\linewidth]{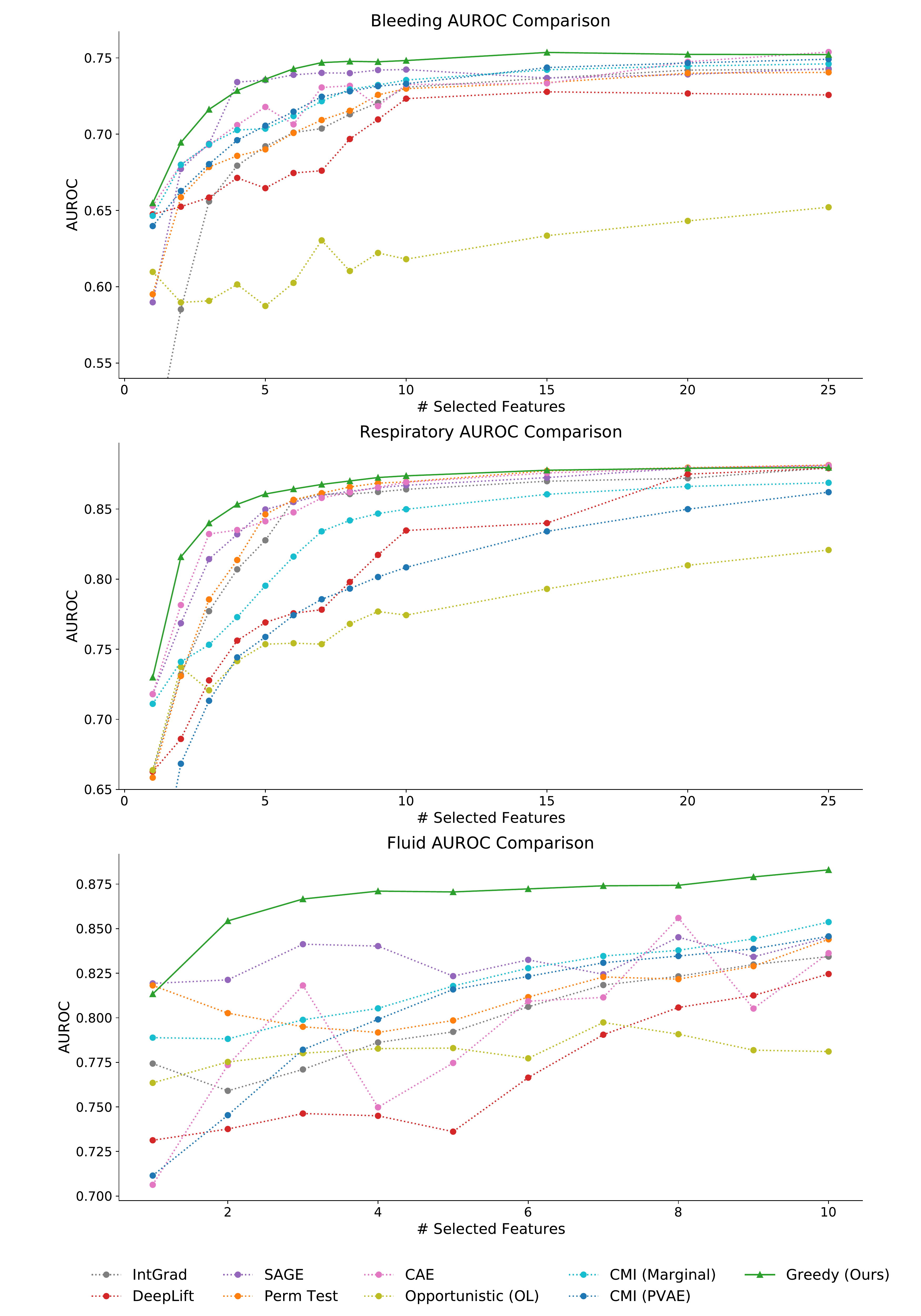}
    \caption{AUROC comparison on the three emergency medicine diagnosis tasks.} \label{fig:medical-large}
\end{figure}

\clearpage
\begin{figure}[t]
    \centering
    \vskip -0.5cm
    \includegraphics[width=0.7\linewidth]{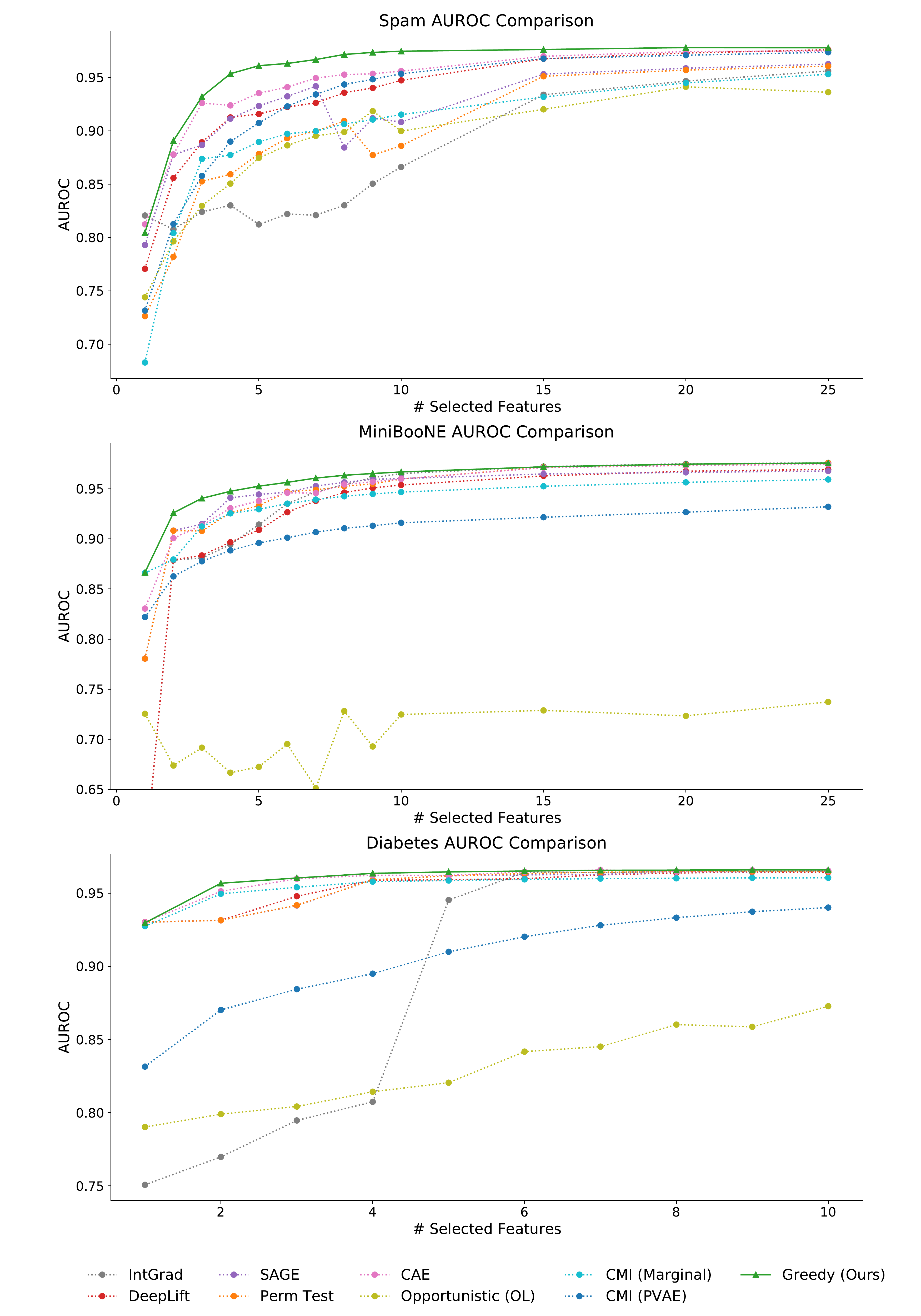}
    \caption{AUROC comparison on the three public tabular datasets.} \label{fig:tabular-large}
\end{figure}

\clearpage
\begin{figure}[t]
    \centering
    \includegraphics[height=5cm]{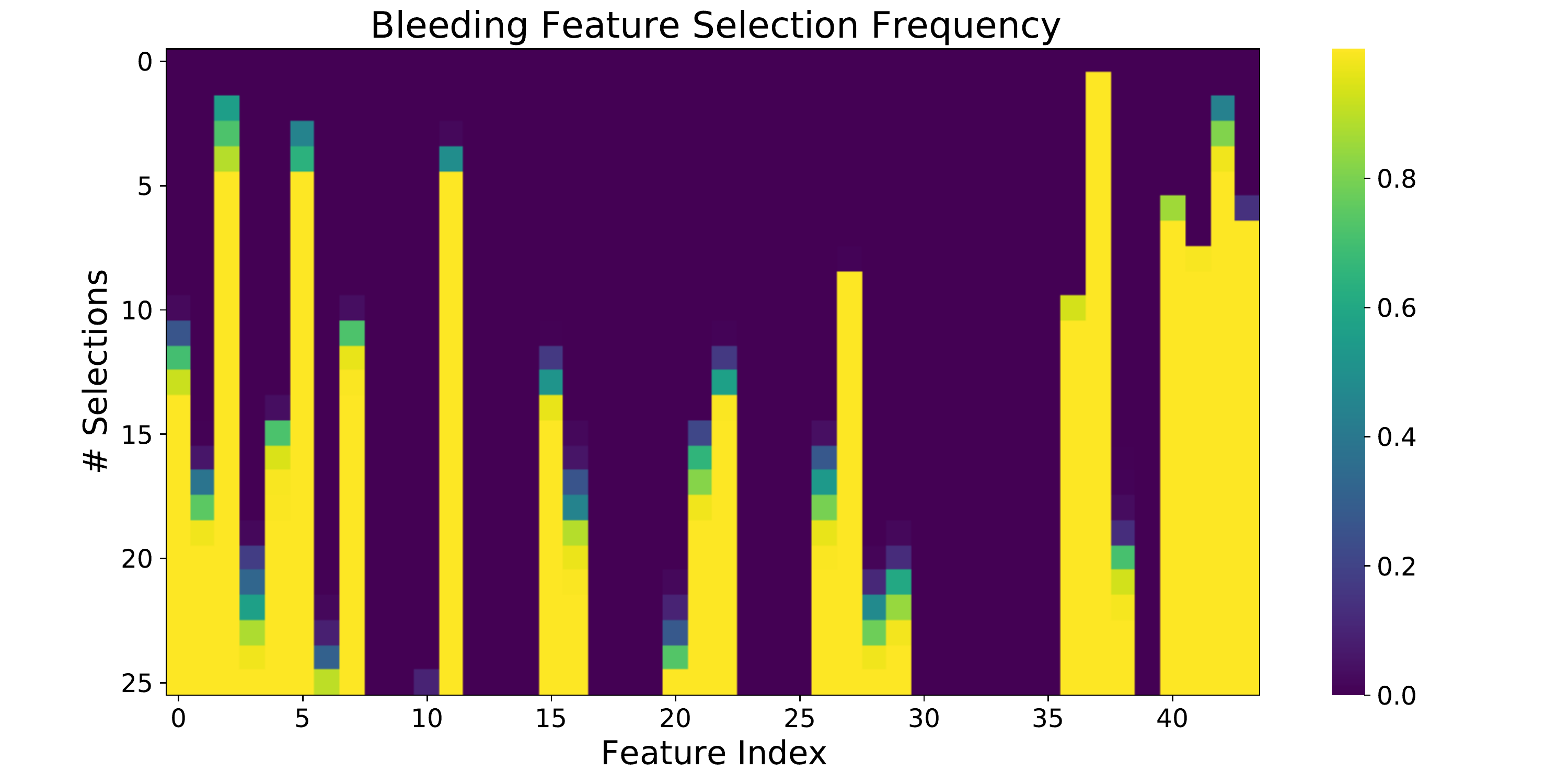}
    \caption{Feature selection frequency for our greedy approach on the bleeding dataset.} \label{fig:fib-freqency}
\end{figure}

\begin{figure}[t]
    \centering
    \includegraphics[height=5cm]{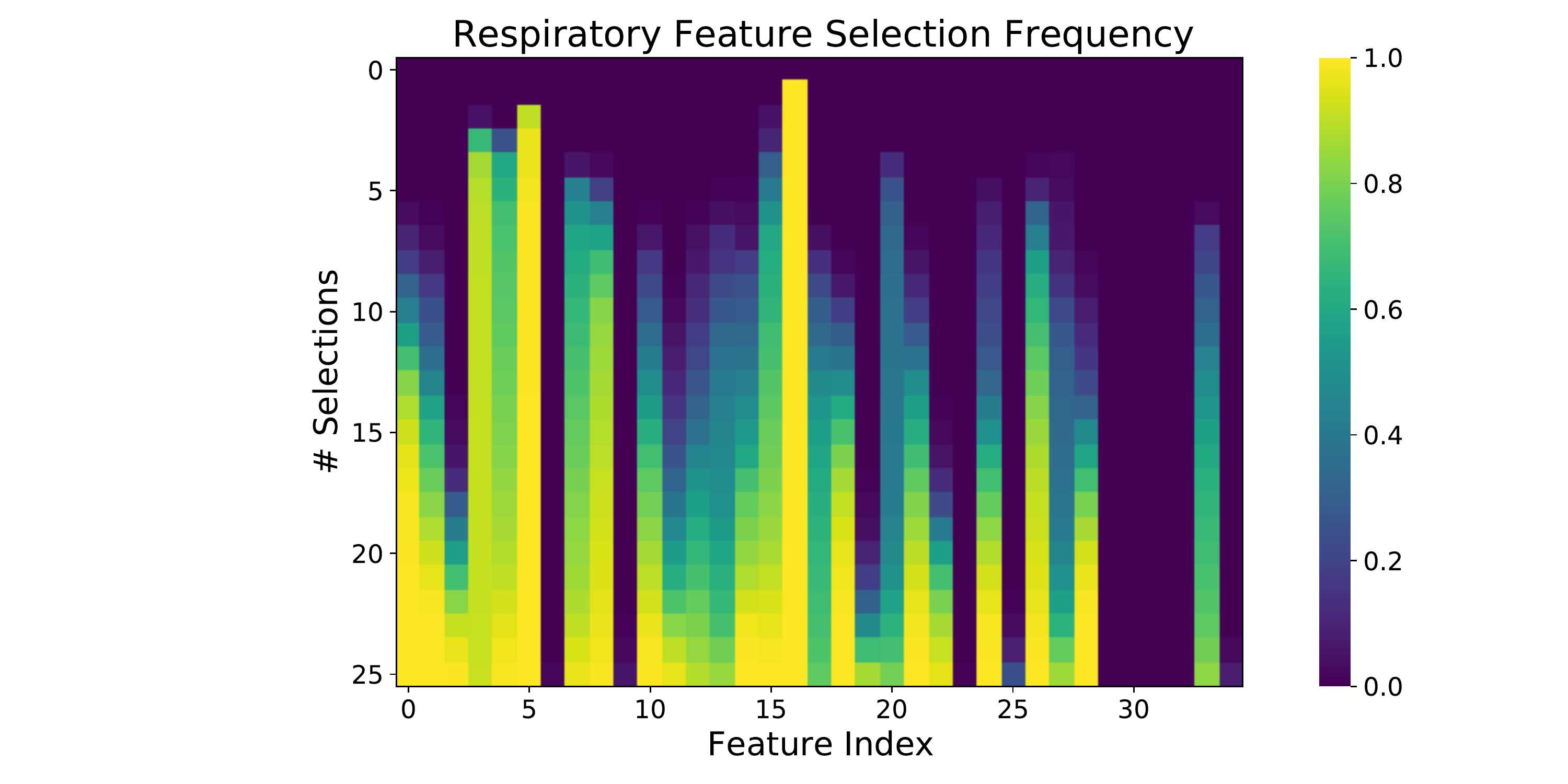}
    \caption{Feature selection frequency for our greedy approach on the respiratory dataset.} \label{fig:intub-freqency}
\end{figure}

\begin{figure}[t]
    \centering
    \includegraphics[height=5cm]{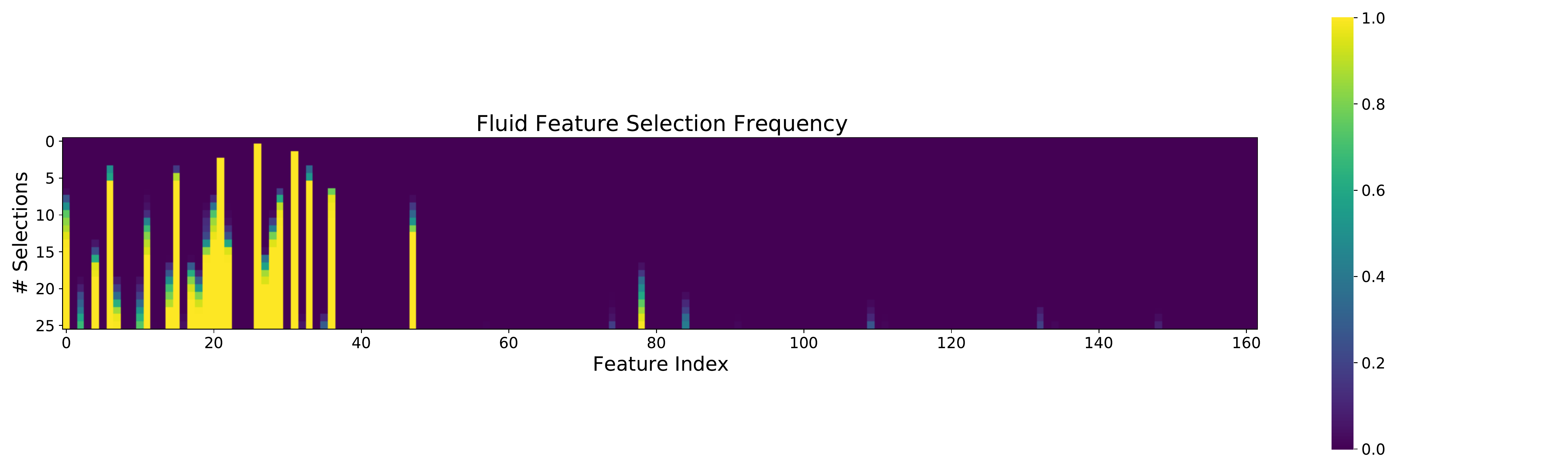}
    \caption{Feature selection frequency for our greedy approach on the fluid dataset.} \label{fig:fluid-freqency}
\end{figure}

\clearpage
\begin{figure}[t]
    \centering
    \includegraphics[height=5cm]{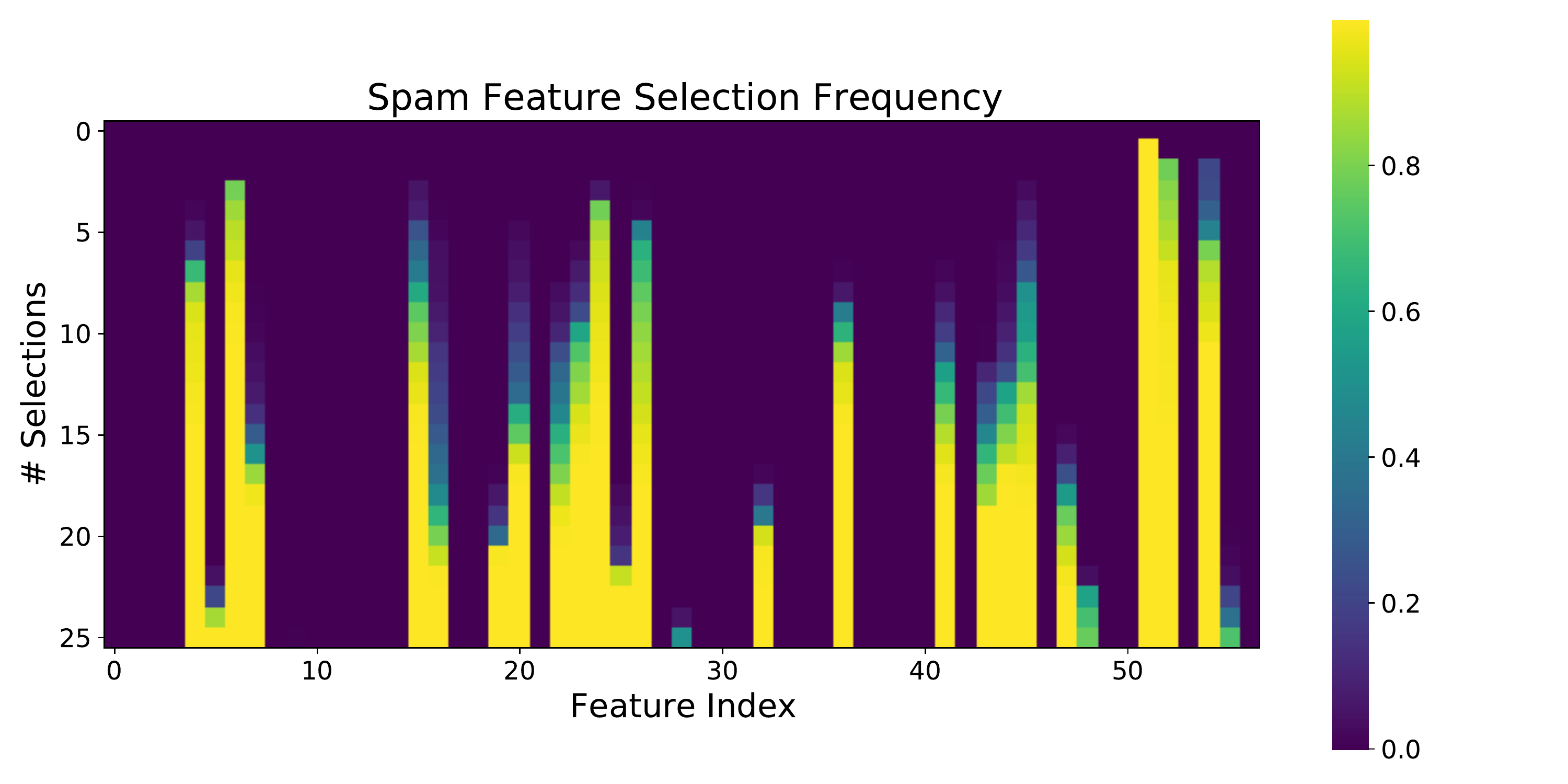}
    \caption{Feature selection frequency for our greedy approach on the spam dataset.} \label{fig:spam-freqency}
\end{figure}

\begin{figure}[t]
    \centering
    \includegraphics[height=5cm]{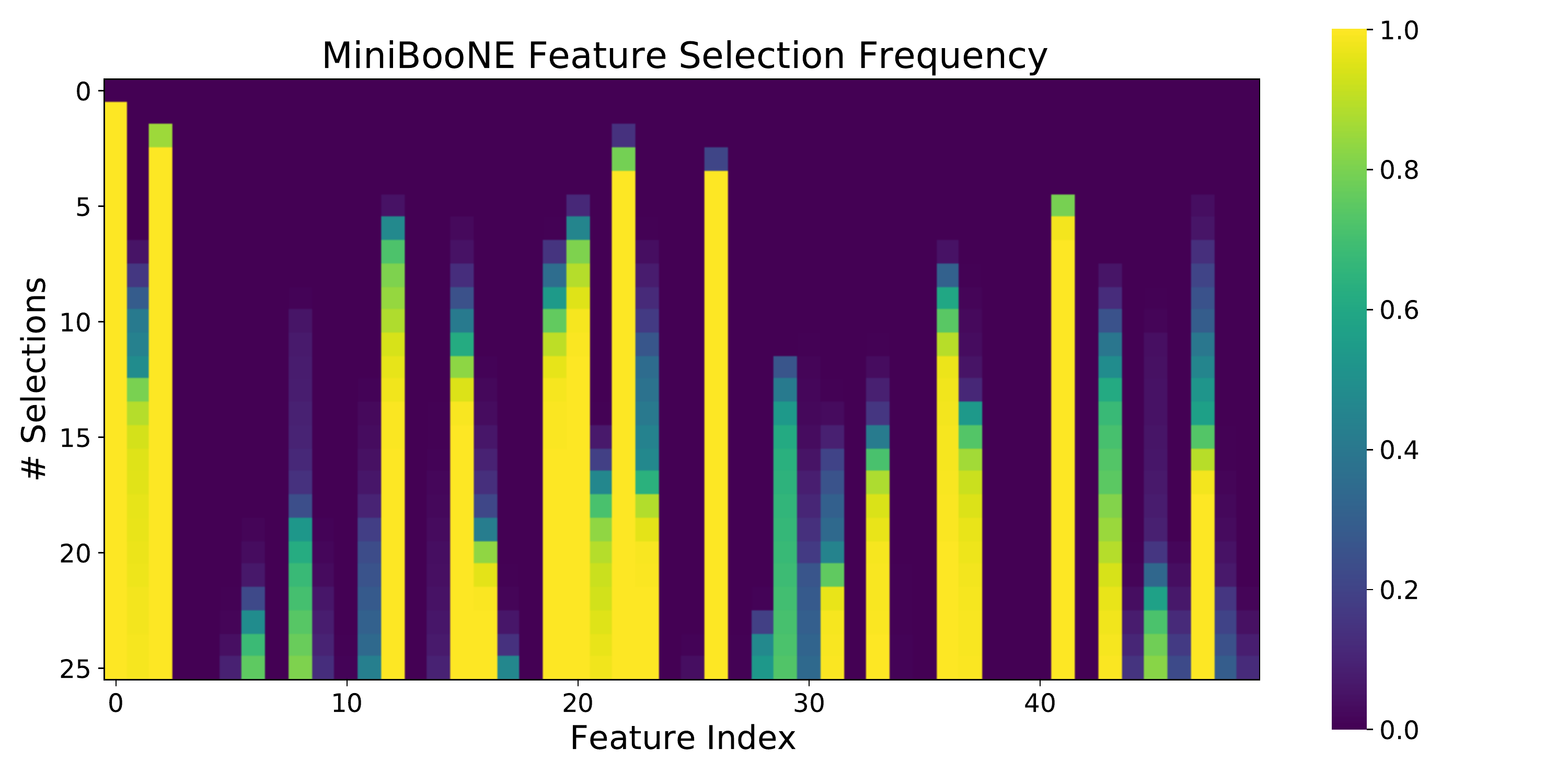}
    \caption{Feature selection frequency for our greedy approach on the MiniBooNE dataset.} \label{fig:miniboone-freqency}
\end{figure}

\begin{figure}[t]
    \centering
    \includegraphics[height=5cm]{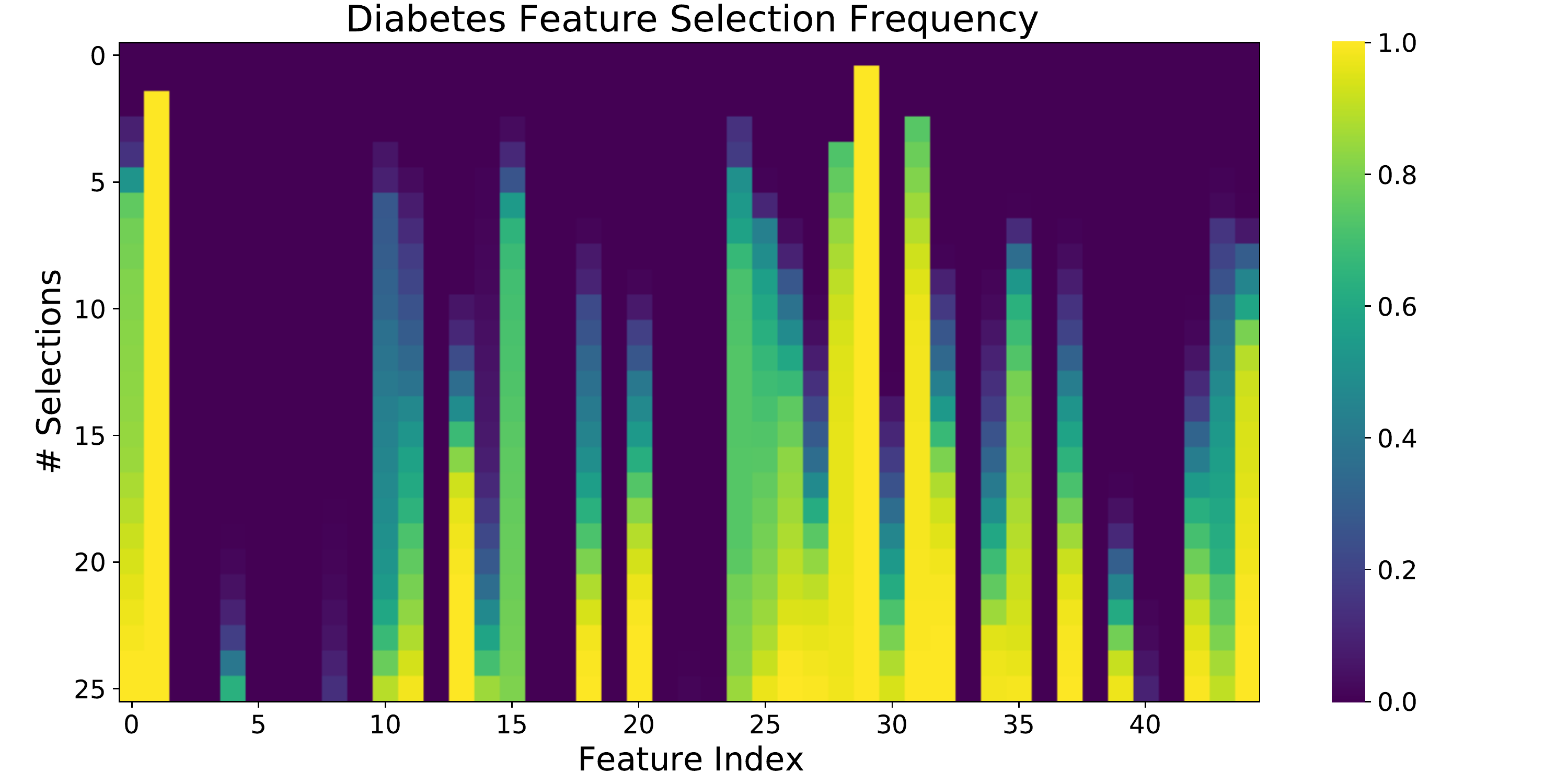}
    \caption{Feature selection frequency for our greedy approach on the diabetes dataset.} \label{fig:diabetes-freqency}
\end{figure}

\clearpage
\begin{figure}[t]
    \centering
    \includegraphics[width=0.85\linewidth]{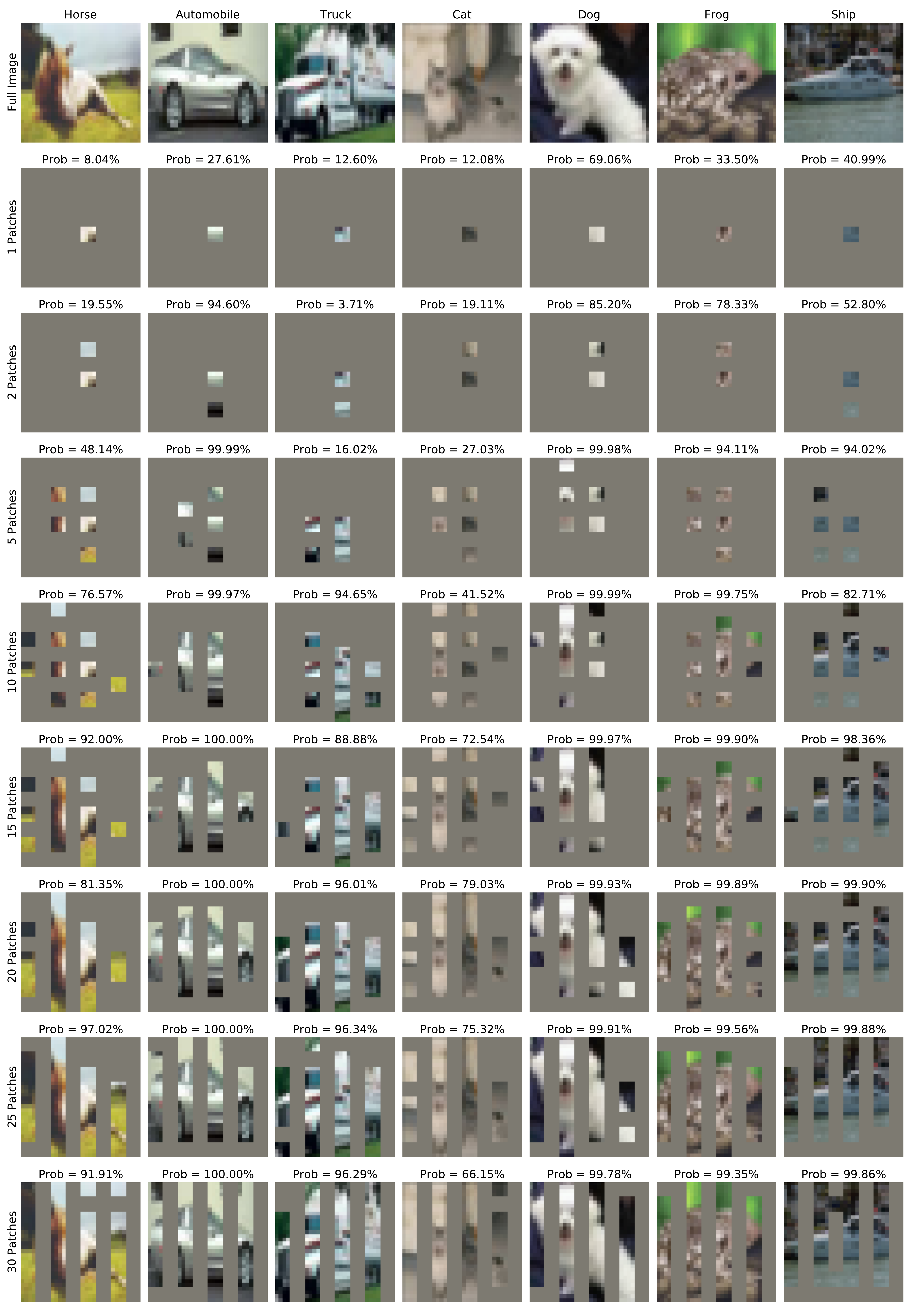}
    \caption{CIFAR-10 predictions with different numbers of patches revealed by our approach.} \label{fig:cifar-large}
\end{figure}

\end{document}